\documentclass[12pt,a4paper]{article}

\usepackage[utf8]{inputenc}
\usepackage[T1]{fontenc}
\usepackage[expansion=false]{microtype}
\usepackage{geometry}
\usepackage{amsmath,amssymb,amsthm}
\usepackage{mathtools}
\usepackage{bm}

\usepackage{graphicx}
\usepackage{booktabs}
\usepackage{subcaption}

\usepackage{algorithm}
\usepackage{algpseudocode}
\usepackage{float}

\usepackage{natbib}

\usepackage[colorlinks=true, linkcolor=blue, citecolor=blue, urlcolor=blue]{hyperref}
\newtheorem{definition}{Definition}[section]
\newtheorem{proposition}[definition]{Proposition}
\newtheorem{remark}[definition]{Remark}
\newtheorem{theorem}[definition]{Theorem}
\newtheorem{lemma}[definition]{Lemma}
\newtheorem{assumption}[definition]{Assumption}
\newtheorem{corollary}[definition]{Corollary}

\newcommand{\X}{\mathcal{X}}
\newcommand{\Z}{\mathcal{Z}}
\newcommand{\E}{\mathcal{E}}
\newcommand{\R}{\mathbb{R}}

\newcommand{\norm}[1]{\left\lVert#1\right\rVert}
\newcommand{\Ogap}{\Omega}
\newcommand{\Dgap}{\Omega^\pi}

\title{Consolidation-Expansion Operator Mechanics:\\
A Unified Framework for Adaptive Learning}

\author{
    Debashis Guha\thanks{S P Jain School of Global Management;
    \texttt{debashis.guha@spjain.org}}
}
\date{}

\begin{document}
\maketitle

\begin{abstract}
Every adaptive learning system must alternate between two operations:
consolidating what it already knows and expanding into new evidence. We propose
\emph{Consolidation-Expansion Operator Mechanics} (OpMech), a framework that
makes this structure precise. The central object is the \emph{order-gap}
$\Ogap(\theta; e)$, the degree to which a consolidation operator~$Q$ and an
expansion operator~$P_e$ fail to commute at a given knowledge state. Because
the order-gap is computable from the system's own trajectory, it serves as
a real-time control signal: large values indicate that the system is still
sensitive to the ordering of consolidation and expansion; once the order-gap
falls and stays small, further processing is unlikely to change the outcome.

Three results give the signal precise meaning: the order-gap
decays along convergent trajectories; a persistently large order-gap implies
the system is far from its settled state; and an order-gap-based
stopping rule terminates with provable guarantees in both noiseless and
bounded-noise settings. The framework applies across five domains:
bandits, reinforcement learning, stochastic optimization, continual learning,
and recursive language models. We give
conditions under which the order-gap reliably tracks convergence in three
representative cases. We develop the recursive language model application in
detail, showing how OpMech replaces heuristic stopping rules and fixed recursion
budgets with principled, evidence-driven alternatives.
\end{abstract}

\noindent\textbf{Keywords.}
adaptive learning; recursive reasoning; iterative refinement; non-commuting operators; 
bandits; reinforcement learning; recursive language models; long context; 
stochastic optimization; continual learning.

\setcounter{tocdepth}{2}
\tableofcontents
\medskip

\section*{Notation}
\addcontentsline{toc}{section}{Notation}
\label{sec:notation}

The following symbols are used throughout. Section references point to the
location where each symbol is first introduced or formally defined; symbols
local to a single domain section are noted as such.

\begin{center}
\renewcommand{\arraystretch}{1.15}
\begin{tabular}{@{}l l l@{}}
\toprule
\textbf{Symbol} & \textbf{Meaning} & \textbf{Reference} \\
\midrule
\multicolumn{3}{@{}l}{\emph{Framework}} \\
$\X$                       & Knowledge state space (Banach space)                          & Def.~\ref{def:knowledge_state} \\
$\theta \in \X$            & Knowledge state                                               & Def.~\ref{def:knowledge_state} \\
$\E$                       & Event space                                                   & §\ref{sec:operators} \\
$e \in \E$                 & Event (incoming evidence)                                     & §\ref{sec:operators} \\
$P_e : \X \to \X$          & Expansion operator (incorporates evidence)                    & Def.~\ref{def:expansion} \\
$Q : \X \to \X$            & Consolidation operator (refines using existing state)         & Def.~\ref{def:consolidation} \\
$\Ogap(\theta; e)$         & Order-gap: $\norm{Q(P_e(\theta)) - P_e(Q(\theta))}$           & Def.~\ref{def:ordergap} \\
$\Dgap(\theta; e)$         & Decision-facing order-gap (composed with decision map $\pi$)  & Eq.~\eqref{eq:ordergap_decision} \\
$\pi : \X \to \Z$          & Decision map                                                  & Eq.~\eqref{eq:ordergap_decision} \\
\addlinespace
\multicolumn{3}{@{}l}{\emph{Theory constants}} \\
$\rho \in [0,1)$           & Consolidation contraction rate                                & Ass.~\ref{ass:contract} \\
$L \geq 0$                 & Expansion Lipschitz constant                                  & Ass.~\ref{ass:lipschitz} \\
$\sigma \geq 0$            & Equilibrium-noise mean envelope, $\E[W_e] \leq \sigma$        & Ass.~\ref{ass:equilibrium} \\
$M \geq \sigma$            & Equilibrium-noise a.s.\ envelope, $W_e \leq M$ a.s.           & Ass.~\ref{ass:equilibrium} \\
$\mu > 0$                  & Order-gap sensitivity coefficient                             & Ass.~\ref{ass:nondeg} \\
$r > 0$                    & Validity radius for non-degeneracy                            & Ass.~\ref{ass:nondeg} \\
$\theta^\star$             & Consolidation fixed point: $Q(\theta^\star) = \theta^\star$   & §\ref{sec:theory_setup} \\
$\theta^\star_\infty$      & Effective equilibrium of joint dynamics                       & Prop.~\ref{prop:effective_eq} \\
$\gamma$                   & Effective contraction rate $\rho L$ (where it appears)        & §\ref{sec:theory_contraction} \\
$\mathcal{F}_t$            & Filtration $\sigma(e_0,\ldots,e_{t-1})$                       & §\ref{sec:theory_setup} \\
\addlinespace
\multicolumn{3}{@{}l}{\emph{Domain-specific (local to indicated section)}} \\
$(\hat\mu_a, \hat\sigma_a^2, n_a)$ & Bandit sufficient statistics for arm $a$               & §\ref{sec:bandits} \\
$K$                        & Number of arms (bandits)                                      & §\ref{sec:bandits} \\
$\lambda, \kappa$          & Bandit prior shrinkage rates                                  & Eq.~\eqref{eq:bandit_Q} \\
$\sigma_r^2$               & Bandit reward variance (distinct from theory $\sigma$)        & §\ref{sec:bandits} \\
$(s, a, r, s')$            & RL transition tuple                                           & §\ref{sec:rl} \\
$w$, $\psi$                & Critic / actor parameters (RL)                                & §\ref{sec:rl} \\
$\eta$                     & SGD learning rate                                             & §\ref{sec:sgd} \\
$\varphi : G \to \R^d$     & RLM epistemic state embedding                                 & §\ref{sec:rlm_domain} \\
$S$                        & RLM aggregated state across chunks                            & §\ref{sec:rlm_domain} \\
\bottomrule
\end{tabular}
\end{center}

\medskip
\noindent\emph{Note.} Where a symbol carries different meanings in
domain-specific sections (e.g., $\sigma$ for the framework's noise
envelope versus $\sigma_r^2$ for bandit reward variance), the
domain-local meaning is in force only within the marked section.

\section{Introduction}
\label{sec:intro}

The tension between knowledge consolidation and knowledge expansion, known as
the exploitation-exploration dilemma in reinforcement learning, is among
the oldest problems in adaptive systems. In multi-armed bandits, the learner
must weigh the arm with the highest estimated reward against trying less-known
ones. In reinforcement learning, the agent must balance policy improvement
against continued evaluation. In stochastic gradient descent, the optimiser
must choose between refining the current basin and escaping to a better one.
In continual learning, the system must protect what it has already learned
while remaining open to new tasks. In recursive language models, the reasoner
must decide between consolidating an answer from evidence already gathered
and recursing into further context.

These problems are typically studied in isolation, each with its own
formalism and analysis tools. Yet beneath all of them lies the same structure:
two operations (one that acts on current knowledge, one that incorporates
new evidence) applied in sequence, with their ordering shaping the learning
trajectory.

This paper proposes \emph{Consolidation-Expansion Operator Mechanics} (OpMech),
a framework that makes this shared structure explicit. The central
observation is:

\begin{quote}
\emph{In every adaptive learning system, consolidation and expansion are
non-commuting operations on the learner's knowledge state. The degree to which
their ordering matters is a measurable, trajectory-dependent quantity, the
order-gap; this quantity can be used as a control signal for adaptive
algorithms.}
\end{quote}

The framework consists of three elements:

\begin{enumerate}
\item A \textbf{stochastic dynamical system} on a knowledge state space $\X$,
      driven by two operators: consolidation $Q$ and expansion $P_e$, with
      canonical dynamics $\theta_{t+1} = Q \circ P_{e_t}(\theta_t)$.

\item The \textbf{order-gap} $\Ogap(\theta; e) = \norm{Q(P_e(\theta)) -
      P_e(Q(\theta))}$, which measures the effect on downstream decisions of
      reordering the two operations.

\item An \textbf{algorithmic principle}: track the order-gap and use it to
      modulate the learning system's behaviour, expand more when $\Ogap$ is
      large, consolidate more when it is small.
\end{enumerate}

Our contributions are:

\begin{itemize}
\item \textbf{A structural sensitivity theorem (Proposition~\ref{prop:local_nondeg},
      Proposition~\ref{prop:noisy_nondeg}, Remark~\ref{rem:second_moment}).}
      The question ``when is the order-gap a faithful convergence signal?''
      reduces to a rank condition on the expected commutator Jacobian
      $\bar\Sigma = \mathbb{E}_e[A B_e - B_e A]$, or, when first-moment effects
      cancel, on the second-moment commutator Gramian
      $G = \mathbb{E}_e[\Sigma_e^\top \Sigma_e]$. Verification reduces to
      checking these matrices at a single point — the consolidation fixed
      point in the noiseless regime, the effective equilibrium of
      Proposition~\ref{prop:effective_eq} in the noisy regime.

\item \textbf{A structural theorem for reinforcement learning
      (Theorem~\ref{prop:rl_rank_deficient},
      Theorem~\ref{prop:rl_nondeg}).}
      Canonical actor-critic consolidation is rank-deficient on the policy
      subspace: the first-moment commutator Jacobian has the entire policy
      subspace in its null space, so the order-gap cannot detect policy
      errors. The framework prescribes an explicit fix (a policy-critic
      consistency term with strength $\beta' > 0$) that restores full rank
      on the locally identifiable actor-critic directions, with order-gap
      sensitivity scaling linearly in $\beta'$.

\item \textbf{Convergence and stopping guarantees
      (Theorem~\ref{thm:contraction},
       Proposition~\ref{prop:lowerbound},
       Theorems~\ref{thm:stopping_det}--\ref{thm:stopping_hp}).}
      The order-gap decays geometrically along convergent trajectories;
      a persistently large order-gap implies suboptimality;
      the stopping rule $\bar\Omega_{t,w} \leq \varepsilon$ terminates in
      finite time with explicit suboptimality bounds in both the noiseless
      ($\sigma = 0$) and bounded-noise ($\sigma > 0$) regimes.
      Corollary~\ref{cor:stop_eff_eq} states the noisy-regime stopping
      guarantees relative to the effective equilibrium
      $\theta^\star_\infty$, which is the actual asymptotic attractor of
      the dynamics. The state-dependent sampling case
      (Theorems~\ref{thm:contraction_cond}--\ref{thm:stopping_cond}) is
      handled by replacing marginal expectation bounds in
      Assumption~\ref{ass:equilibrium} with uniform conditional bounds; the
      i.i.d.\ theorems then carry over verbatim with no mixing-time
      correction.

\item \textbf{A unified framework for adaptive learning.}
      The order-gap is a single, measurable, domain-independent quantity
      that organises five adaptive learning settings: bandits,
      reinforcement learning, stochastic gradient descent, continual
      learning, and recursive language models.
      Three settings (bandits, RLMs, linear actor-critic RL) receive
      explicit verification via the structural sensitivity theorem; SGD
      reduces to classical Robbins-Monro stochastic approximation in the
      $Q = \mathrm{id}$ limit, with the order-gap a non-trivial diagnostic
      for adaptive variants; continual learning is a qualitative mapping
      (Section~\ref{sec:domains}).

\item \textbf{Applications to recursive language models
      (Section~\ref{sec:recursive_intelligence}).}
      The order-gap supplies three principled control signals for recursive
      reasoning — stopping, recursion scheduling, adaptive extraction —
      replacing heuristic alternatives. Empirical validation is reported
      in a companion paper.
\end{itemize}

We develop the framework in Section~\ref{sec:framework}, state the
algorithmic principle in Section~\ref{sec:principle}, establish formal
convergence and stopping guarantees in Section~\ref{sec:theory}, demonstrate
the framework across five domains in Section~\ref{sec:domains}, develop the
recursive intelligence application in detail in
Section~\ref{sec:recursive_intelligence}, show how existing algorithms
implicitly control $\Ogap$ in Section~\ref{sec:hyperparams}, and discuss the
research programme opened by the framework in Section~\ref{sec:discussion}.

\section{The OpMech Framework}
\label{sec:framework}

\subsection{Knowledge State}
\label{sec:state}

\begin{definition}[Knowledge state]
\label{def:knowledge_state}
The \emph{knowledge state} of an adaptive learner is a point $\theta \in \X$,
where $\X$ is a state space encoding everything the learner currently knows.
\end{definition}

The precise structure of $\X$ depends on the domain:

\begin{center}
\begin{tabular}{ll}
\toprule
\textbf{Domain} & \textbf{Knowledge state $\theta$} \\
\midrule
Multi-armed bandits     & Sufficient statistics $\{(\hat\mu_a, \hat\sigma_a^2, n_a)\}_{a=1}^K$ \\
Stochastic gradient descent & Model parameters + optimiser state \\
Reinforcement learning  & Value function parameters $\theta$ of $Q_\theta(s,a)$ \\
Continual learning      & Task-specific parameters across all tasks seen \\
Recursive language models & Aggregated state $S$ over chunks processed so far \\
\bottomrule
\end{tabular}
\end{center}

\subsection{Two Operators}
\label{sec:operators}

The learner's trajectory through $\X$ is governed by two operators.

\begin{definition}[Expansion operator]
\label{def:expansion}
For an event $e \in \E$, the \emph{expansion operator}
$P_e : \X \to \X$ is defined by $P_e(\theta) = U_e(\theta)$,
where $U_e$ is the evidence update corresponding to event $e$.
\end{definition}

An \emph{event} is whatever arrives from outside the learner: a reward
observation in bandits, a sampled minibatch in SGD, a transition in RL,
a chunk of context in recursive reasoning. The
key property is that $P_e$ brings information into $\theta$ that was not
already there. The subscript $e$ matters: the family $\{P_e : e \in
\E\}$ is indexed by the event because a different event induces a different
update, and this distinction is what makes the expected commutator
Jacobian $\mathbb{E}_e[AB_e - B_eA]$ non-trivial, with each $e$ contributing a distinct
$B_e = DP_e(\theta^\star)$.

\begin{definition}[Consolidation operator]
\label{def:consolidation}
The \emph{consolidation operator} $Q : \X \to \X$ refines $\theta$ toward the
greedy optimum using only information already present in $\theta$. No external
evidence enters.
\end{definition}

The defining properties of $Q$ are:
\begin{enumerate}
\item \textbf{Endogenous}: $Q$ depends only on $\theta$, not on any event $e$.
\item \textbf{Contractive}: $Q$ concentrates the state around the current
      greedy optimum, reducing uncertainty about the optimal action.
\item \textbf{Fixed-point structure}: the fixed points of $Q$ are knowledge
      states where pure consolidation is self-consistent.
\end{enumerate}

The asymmetry between the two operators is fundamental: $P_e$ requires
external input; $Q$ does not.

\subsection{Dynamics}
\label{sec:dynamics}

\begin{definition}[Stochastic dynamical system]
\label{def:dynamics}
The \emph{OpMech dynamics} are the discrete-time Markov chain on $\X$ with
canonical trajectory
\begin{equation}
\label{eq:dynamics}
\theta_{t+1} = Q \circ P_{e_t}(\theta_t), \qquad e_t \sim P(\cdot \mid \theta_t),
\end{equation}
and transition kernel
\begin{equation}
\label{eq:kernel}
\mathbb{P}(\theta_{t+1} \in A \mid \theta_t) =
\int_{\E} \mathbf{1}_A\bigl(Q(U_e(\theta_t))\bigr) \, P(de \mid \theta_t).
\end{equation}
\end{definition}

At each step, the learner expands (incorporates evidence from event $e_t$) and
then consolidates (refines toward the greedy optimum). The randomness
comes from $e_t$, which may depend on the current knowledge
state (as in bandits, where the arm-selection policy depends on $\theta$) or
may be independent of it (as in SGD, where the minibatch is drawn from a fixed
distribution).

\subsection{The Order-Gap}
\label{sec:ordergap}

\begin{definition}[Order-gap]
\label{def:ordergap}
The \emph{order-gap} at state $\theta$ given event $e$ is
\begin{equation}
\label{eq:ordergap}
\Ogap(\theta; e) = d\bigl(Q(P_e(\theta)),\; P_e(Q(\theta))\bigr),
\end{equation}
where $d$ is a metric on $\X$. When $\X$ is a vector space this reduces to
\[
\Ogap(\theta; e) = \norm{Q(P_e(\theta)) - P_e(Q(\theta))}.
\]
\end{definition}

The order-gap measures the distance between two trajectories: one where the
learner expands before consolidating, and one where it consolidates before
expanding. When $\Ogap$ is large, the ordering matters; the learning system
has not settled. When $\Ogap$ is small, the ordering is nearly
irrelevant; the system has converged.

\textbf{Decision-facing order-gap.}
If a decision map $\pi : \X \to \Z$ is specified and $d_\Z$ is a metric on $\Z$,
the order-gap at the decision level is:
\begin{equation}
\label{eq:ordergap_decision}
\Dgap(\theta; e)
= d_\Z\!\bigl(\pi(Q(P_e(\theta))),\; \pi(P_e(Q(\theta)))\bigr).
\end{equation}
When $\Z$ is a vector space, this reduces to the norm of the difference.

\subsection{Key Properties}
\label{sec:properties}

\textbf{Non-commutativity is generically present.}
In nontrivial adaptive systems, consolidation and expansion generically need
not commute: $Q$ contracts toward a fixed point that depends on the current
state, while $P_e$ moves the state, so the two operations interact
differently depending on their order. The order-gap measures the extent to
which they fail to commute along the trajectory. Non-stationarity makes the
order-gap \emph{persistently large}, but the non-commutativity is a property
of the operator pair, not a consequence of the environment.

\textbf{Convergence under stationarity (noiseless case).}
In the noiseless stationary regime ($\sigma = 0$, see Assumption~\ref{ass:equilibrium}):
\begin{equation}
\label{eq:convergence}
\mathbb{E}\bigl[\Ogap(\theta_t; e_t)\bigr] \to 0 \quad \text{as } t \to \infty.
\end{equation}
In noisy stationary regimes ($\sigma > 0$), the expected order-gap is bounded
by a noise-floor term quantified in Section~\ref{sec:theory_noisefloor}.
This is not the same as $[Q, P_e] = 0$; the operators remain structurally
non-commutative. Only their expected trajectory impact shrinks.

\textbf{Persistence under non-stationarity.}
When the environment changes, $\mathbb{E}[\Ogap(\theta_t; e_t)]$ remains
large or grows. The order-gap detects that the learning system has not settled
relative to the shifting environment.

\section{The Algorithmic Principle}
\label{sec:principle}

The core claim of OpMech is not that the order-gap exists (it trivially does
for any pair of non-commuting maps), but that $\Ogap$ is the right
quantity to \emph{track and control} in adaptive systems.

The algorithmic principle is:

\begin{quote}
\emph{Compute the order-gap $\Ogap(\theta_t; e_t)$ along the learning
trajectory. Use it as a control signal: expand more when $\Ogap$ is large,
consolidate more when $\Ogap$ is small.}
\end{quote}

This principle takes three concrete forms:

\textbf{Adaptive exploration rate.}
Let $\eta_t$ be the exploration rate at step $t$. Set
$\eta_t = \phi(\Ogap(\theta_t; e_t))$, where $\phi$ is a monotone increasing
function. When the order-gap is large, the ordering of consolidation and
expansion matters greatly: the learner has not settled, and it should expand
more. When the order-gap is small, the learner has nearly converged, so it
should consolidate.

\textbf{Consolidation scheduling.}
Rather than consolidating at every step, consolidate when the accumulated
order-gap since the last consolidation exceeds a threshold $c$:
\[
\sum_{i=\tau}^{t} \Ogap(\theta_i; e_i) > c,
\]
where $\tau$ is the last consolidation step.

\textbf{Stopping criterion.}
Declare convergence when a windowed average of $\Ogap$ falls below a
threshold $\varepsilon$. Section~\ref{sec:theory} establishes formal
convergence and stopping guarantees for the order-gap signal. The adaptive
exploration and consolidation-scheduling rules above are algorithmic design
principles motivated by this theory; their performance guarantees are
domain-specific and developed in companion work.

\section{Convergence, Bounds, and Stopping Guarantees}
\label{sec:theory}

The algorithmic principle of Section~\ref{sec:principle} rests on three
claims that need formal backing: that the order-gap decays along convergent
trajectories, that a persistently large order-gap signals genuine
suboptimality, and that the stopping rule $\Ogap < \varepsilon$ terminates
in finite time with bounded error. This section establishes each of these.

\subsection{Setting and Assumptions}
\label{sec:theory_setup}

\textbf{Scope.}
The formal results in this section are proved first for the case where
events $e_t$ are drawn independently from a fixed distribution $P$ on
$(\E, \mathcal{F})$, independent of $\theta_t$; we refer to this as the
\emph{fixed-sampling setting}. This covers independent evidence streams and
settings where the sampling policy is held fixed. When the event distribution
depends on the current state (as in fully adaptive bandits or off-policy
RL), Section~\ref{sec:state_dependent} extends the analysis to the
\emph{state-dependent setting} under uniform conditional moment bounds; the
contraction and stopping theorems carry over verbatim, because the
underlying concentration tool is martingale concentration on residuals
$\Ogap(\theta_s; e_s) - \mathbb{E}[\Ogap(\theta_s; e_s) \mid \mathcal{F}_s]$,
which preserves its martingale-difference structure under state-dependent
sampling.

We work on a Banach space $(\X, \norm{\cdot})$. The operators $Q: \X \to \X$
and $\{P_e : e \in \E\}$ generate the dynamics of Equation~\eqref{eq:dynamics},
with events $e_t$ drawn i.i.d.\ from $P$. Denote
$\mathcal{F}_t := \sigma(e_0, \ldots, e_{t-1})$.

\begin{assumption}[Contractive consolidation]
\label{ass:contract}
There exists $\rho \in [0, 1)$ such that
\[
\norm{Q(\theta_1) - Q(\theta_2)} \leq \rho \norm{\theta_1 - \theta_2}
\]
for all $\theta_1, \theta_2 \in \X$.
\end{assumption}

By the Banach fixed-point theorem, $Q$ admits a unique fixed point
$\theta^\star \in \X$, the \emph{consolidation fixed point}.

\begin{assumption}[Lipschitz expansion]
\label{ass:lipschitz}
There exists $L \geq 0$ such that
\[
  \norm{P_e(\theta_1) - P_e(\theta_2)} \leq L\,\norm{\theta_1 - \theta_2}
\]
for all $\theta_1, \theta_2 \in \X$ and all $e \in \E$.
\end{assumption}

\begin{assumption}[Equilibrium regularity]
\label{ass:equilibrium}
The random variable $W_e := \norm{P_e(\theta^\star) - \theta^\star}$
satisfies $\mathbb{E}[W_e] \leq \sigma$ for some $\sigma \geq 0$. For
high-probability statements, we further assume $W_e \leq M$ almost surely,
for some $M \geq \sigma$.
\end{assumption}

The case $\sigma = 0$ implies $W_e = 0$ a.s., i.e.,
$P_e(\theta^\star) = \theta^\star$ almost surely, placing us in the
\emph{noiseless regime}. This regime admits deterministic
(pathwise) bounds.

\begin{assumption}[Local order-gap sensitivity]
\label{ass:nondeg}
There exist $r > 0$ and $\mu > 0$ such that, for all $\theta \in \X$ with
$\norm{\theta - \theta^\star} \leq r$,
\[
  \mathbb{E}_e\bigl[\Ogap(\theta; e)\bigr] \;\geq\; \mu \norm{\theta - \theta^\star}.
\]
Informally: within a neighbourhood of $\theta^\star$ of radius $r$, the
expected order-gap is bounded below by a positive multiple of the distance to
the fixed point, so a large order-gap reliably signals that the system has not
settled. In the noisy regime ($\sigma > 0$ in
Assumption~\ref{ass:equilibrium} with $W_e \leq M$ a.s.), the validity radius
must satisfy $r > \rho M/(1-\gamma)$ — the asymptotic radius of the
trajectory's noise ball, established in Lemma~\ref{lem:containment} below —
otherwise the trajectory need not remain within the validity region and the
stopping guarantees do not bind. This is the non-degeneracy condition used
by the stopping theorems.
\end{assumption}

Assumption~\ref{ass:nondeg} is now stated locally because the natural
verification strategy is local linearisation at $\theta^\star$.
Proposition~\ref{prop:local_nondeg} gives a sufficient Jacobian rank
condition establishing the assumption under $\sigma = 0$, and
Proposition~\ref{prop:noisy_nondeg} extends this to the stochastic case via
the effective equilibrium. The stopping theorems
(Theorems~\ref{thm:stopping_det}--\ref{thm:stopping_hp}) require the
trajectory to lie within the validity ball $B_r(\theta^\star)$ during the
relevant stopping window. In the noiseless regime,
Theorem~\ref{thm:contraction}(iii) gives geometric decay
$\norm{\theta_t - \theta^\star} \leq \gamma^t R_0$ a.s., so the trajectory
enters $B_r$ deterministically after
$T_0^{\mathrm{det}}(r) := \lceil \log(R_0/r)/\log(1/\gamma) \rceil$ steps. In
the noisy regime, Lemma~\ref{lem:containment} establishes the analogous
deterministic containment provided $r > \rho M/(1-\gamma)$. The domain
verifications in Section~\ref{sec:theory_domains} establish local
applicability via Proposition~\ref{prop:local_nondeg} or
Proposition~\ref{prop:noisy_nondeg}, not global validity.

\begin{remark}[Compatibility of the validity radius with the noise floor]
\label{rem:radius_compat}
Propositions~\ref{prop:local_nondeg} and~\ref{prop:noisy_nondeg} establish
Assumption~\ref{ass:nondeg} on a local ball whose radius is determined by
the second-derivative bounds of the operators: $r = \mu_0/(2R)$ in the
noiseless case, $r_\infty = \mu_0^\infty/(2R)$ in the noisy case. The
stopping theorems require this radius to exceed the asymptotic noise ball
radius $\rho M/(1-\gamma)$ established by
Lemma~\ref{lem:containment}. The corresponding compatibility condition is
\begin{equation}
\label{eq:compat}
  \frac{\mu_0}{2R} \;>\; \frac{\rho M}{1-\gamma}
  \qquad\Longleftrightarrow\qquad
  \mu_0 \;>\; \frac{2R \rho M}{1-\gamma},
\end{equation}
which says the linear commutator signal $\mu_0$ must dominate the
second-order remainder evaluated at the noise-ball scale. In the noiseless
regime ($M = 0$), \eqref{eq:compat} is automatic. In the noisy regime, it
is the operational requirement that the framework's signal-to-noise
structure is non-degenerate. The same condition with $\mu_0 \to \mu_0^\infty$
and $M \to M^\infty$ governs the noisy-regime stopping bounds centred at
$\theta^\star_\infty$.
\end{remark}

\subsection{When the Order-Gap Tracks Distance to Equilibrium}
\label{sec:local_nondeg}

Assumption~\ref{ass:nondeg} is a global, first-moment condition on the
order-gap. We establish a local version that reduces its verification to a
standard rank condition on an averaged matrix.

Assume $Q$ and $\{P_e\}$ are $C^2$ in a neighbourhood of $\theta^\star$, with
second derivatives of $Q$ and $P_e$ uniformly bounded in that neighbourhood
(or admitting an integrable dominating function over $e$). Let
\[
  A := DQ(\theta^\star), \qquad B_e := DP_e(\theta^\star),
\]
and define the \emph{commutator Jacobian} at the equilibrium:
\[
  \Sigma_e := A B_e - B_e A.
\]
Under $\sigma = 0$, a first-order Taylor expansion at
$\theta = \theta^\star + x$ for small $x$ gives:
\[
  \Ogap(\theta^\star + x;\, e)
    = \norm{Q(P_e(\theta^\star + x)) - P_e(Q(\theta^\star + x))}
    = \norm{\Sigma_e\, x + r_e(x)},
\]
where $\norm{r_e(x)} = O(\norm{x}^2)$ uniformly in $e$ (by the uniform
$C^2$ regularity assumption and Assumption~\ref{ass:lipschitz}). Taking
expectation over $e$ and applying Jensen's inequality:
\[
  \mathbb{E}_e[\Ogap(\theta^\star + x; e)]
    \;\geq\; \norm{\mathbb{E}_e[\Sigma_e] \, x} - O(\norm{x}^2),
\]
where the $O(\norm{x}^2)$ constant depends on the second-derivative bounds
of $Q$ and the integrable dominating function for $\{P_e\}$.

\begin{proposition}[Local order-gap sensitivity via Jacobian rank]
\label{prop:local_nondeg}
Assume $\sigma = 0$, and that $Q, P_e \in C^2$ near $\theta^\star$ with
second derivatives uniformly bounded in that neighbourhood (or with an
integrable dominating function for the second derivatives over $e$). Define
$\bar\Sigma := \mathbb{E}_e[\Sigma_e]$ and let
$\mu_0 := \sigma_{\min}(\bar\Sigma)$ denote its smallest singular value.
If $\mu_0 > 0$, then there exists $r > 0$ such that for all $\theta$ with
$\norm{\theta - \theta^\star} \leq r$,
\[
  \mathbb{E}_e[\Ogap(\theta; e)] \;\geq\; \tfrac{\mu_0}{2} \, \norm{\theta - \theta^\star}.
\]
In particular, Assumption~\ref{ass:nondeg} holds locally with constant
$\mu = \mu_0 / 2$.
\end{proposition}

\begin{proof}
By Jensen and Taylor,
$\mathbb{E}_e[\Ogap(\theta^\star + x; e)]
\geq \norm{\bar\Sigma x} - R \norm{x}^2$
for some $R > 0$ depending on the second-derivative bounds of $Q$ and the
integrable dominating function for $\{P_e\}$. Since $\norm{\bar\Sigma x} \geq \mu_0
\norm{x}$, the right-hand side is at least $\mu_0 \norm{x} - R \norm{x}^2$.
Choosing $r := \mu_0 / (2R)$ and restricting to $\norm{x} \leq r$ gives
the result. \qedhere
\end{proof}

In other words, order-gap sensitivity reduces to an invertibility condition on
the \emph{expected commutator Jacobian} $\bar\Sigma = \mathbb{E}_e[A B_e - B_e A]$.
Verifying this requires only the Jacobians of $Q$ and $P_e$ at a single point,
with no global argument needed. In finite-dimensional settings (or after
choosing a Hilbert coordinate chart on $\X$), the smallest singular value
$\mu_0 = \sigma_{\min}(\bar\Sigma)$ is well-defined; on a general Banach
space one replaces the singular-value condition with an equivalent lower bound
on $\norm{\bar\Sigma x}$.

\begin{remark}[Second-moment coverage and first-moment sensitivity]
\label{rem:second_moment}
An alternative sufficient condition uses the second-moment matrix of the local
commutator, $G := \mathbb{E}_e[\Sigma_e^\top \Sigma_e]$, which we call the
\emph{commutator Gramian}. If
$\lambda_{\min}(G) \geq \mu_1^2 > 0$, then locally
$\mathbb{E}_e[\Ogap(\theta; e)^2] \geq \mu_1^2 \norm{\theta -
\theta^\star}^2 + O(\norm{\theta - \theta^\star}^3)$. This Gramian condition
is often easier to verify than the first-moment condition $\mu_0 > 0$ on
$\bar\Sigma$, because first-moment effects can cancel across $e$ while
second-moment coverage cannot.

\textbf{From second moment to first moment.} If additionally there is a uniform local envelope
$\Ogap(\theta^\star + x; e) \leq C\norm{x}$ for all $e$ (which holds whenever
the Jacobians $A$ and $B_e$ are uniformly bounded), then shrinking the
neighbourhood if necessary, the remainder in the Gramian expansion is
dominated and
\[
  \mathbb{E}_e[\Ogap(\theta^\star + x; e)^2]
  \;\geq\; \frac{\mu_1^2}{2}\norm{x}^2.
\]
Together with $\Ogap \leq C\norm{x}$, dividing
$\tfrac{\mu_1^2}{2}\norm{x}^2 \leq \mathbb{E}[\Ogap^2] \leq
C\norm{x}\,\mathbb{E}[\Ogap]$ gives:
\[
  \mathbb{E}_e[\Ogap(\theta^\star + x; e)]
  \;\geq\; \frac{\mu_1^2}{2C}\,\norm{x}.
\]
Thus the order-gap lower bound holds locally with $\mu = \mu_1^2/(2C)$. This argument is used
in the RLM verification below.
\end{remark}

\subsection{Effective Equilibrium and Noisy Local Sensitivity}
\label{sec:effective_equilibrium}

Proposition~\ref{prop:local_nondeg} requires the noiseless hypothesis
$\sigma = 0$, in which case $P_e(\theta^\star) = \theta^\star$ a.s.\ and the
Taylor expansion at $\theta^\star$ has no zeroth-order term. In stochastic
domains (bandits, RL), this hypothesis fails: the raw expansion
operator does not fix $\theta^\star$, and the dynamics converge not to
$\theta^\star$ but to a distinct effective equilibrium that absorbs the
asymptotic bias. We make this precise so that the Jacobian calculations in
Section~\ref{sec:theory_domains} for stochastic domains are formally
grounded rather than informally read as ``expected-update.''

\begin{proposition}[Existence of effective equilibrium]
\label{prop:effective_eq}
Under Assumptions~\ref{ass:contract}--\ref{ass:lipschitz} with
$\gamma = \rho L < 1$, the expected one-step dynamics map
\[
  T(\theta) \;:=\; \mathbb{E}_e\bigl[Q(P_e(\theta))\bigr]
\]
is a contraction on $(\X, \norm{\cdot})$ with rate $\gamma$, and admits a
unique fixed point $\theta^\star_\infty \in \X$, which we call the
\emph{effective equilibrium}. In the noiseless regime ($W_e = 0$ a.s.),
$\theta^\star_\infty = \theta^\star$.
\end{proposition}

\begin{proof}
For any $\theta_1, \theta_2 \in \X$:
\[
  \norm{T(\theta_1) - T(\theta_2)}
  \leq \mathbb{E}_e\bigl[\norm{Q(P_e(\theta_1)) - Q(P_e(\theta_2))}\bigr]
  \leq \rho L \norm{\theta_1 - \theta_2}
  = \gamma \norm{\theta_1 - \theta_2},
\]
using Jensen's inequality and Assumptions~\ref{ass:contract}--\ref{ass:lipschitz}.
Since $\gamma < 1$ and $(\X, \norm{\cdot})$ is complete, the Banach
fixed-point theorem yields a unique $\theta^\star_\infty$ with
$T(\theta^\star_\infty) = \theta^\star_\infty$. Under $W_e = 0$ a.s.,
$P_e(\theta^\star) = \theta^\star$ a.s., so
$T(\theta^\star) = \mathbb{E}_e[Q(\theta^\star)] = Q(\theta^\star) = \theta^\star$,
and $\theta^\star_\infty = \theta^\star$ by uniqueness. \qedhere
\end{proof}

The effective equilibrium $\theta^\star_\infty$ is the actual asymptotic
attractor of the expected dynamics. The displacement
$\theta^\star_\infty - \theta^\star$ is the systematic bias from
$\mathbb{E}_e[P_e(\theta^\star) - \theta^\star] \neq 0$, of order
$\rho\sigma/(1-\gamma)$ (the trajectory drift term identified in
Section~\ref{sec:theory_noisefloor}). At $\theta^\star_\infty$, the
order-gap retains a residual noise level
$\sigma^\infty := \mathbb{E}_e[\norm{P_e(\theta^\star_\infty) -
\theta^\star_\infty}]$, generally smaller than $\sigma$ but not necessarily
zero.

\begin{proposition}[Noisy local order-gap sensitivity at the effective equilibrium]
\label{prop:noisy_nondeg}
Assume $Q$ and $\{P_e\}$ are $C^2$ in a neighbourhood of
$\theta^\star_\infty$ with second derivatives uniformly bounded (or with an
integrable dominating function over $e$), and that
$\mathbb{E}_e[\norm{P_e(\theta^\star_\infty) - \theta^\star_\infty}^2] \leq
(\sigma^\infty)^2 < \infty$. Define
\[
  A_\infty := DQ(\theta^\star_\infty), \qquad
  B_e^\infty := DP_e(\theta^\star_\infty), \qquad
  \Sigma_e^\infty := A_\infty B_e^\infty - B_e^\infty A_\infty,
\]
$\bar\Sigma_\infty := \mathbb{E}_e[\Sigma_e^\infty]$, and
$\mu_0^\infty := \sigma_{\min}(\bar\Sigma_\infty)$. If $\mu_0^\infty > 0$,
there exist $r_\infty > 0$ and $C_\infty > 0$ such that for all $\theta$ with
$\norm{\theta - \theta^\star_\infty} \leq r_\infty$:
\[
  \mathbb{E}_e[\Ogap(\theta; e)]
  \;\geq\; \tfrac{\mu_0^\infty}{2} \norm{\theta - \theta^\star_\infty}
  \;-\; C_\infty\, \sigma^\infty.
\]
Equivalently, Assumption~\ref{ass:nondeg} (with $\theta^\star$ replaced by
$\theta^\star_\infty$) holds in the form
$\mathbb{E}_e[\Ogap(\theta;e)] \geq (\mu_0^\infty/4)\norm{\theta -
\theta^\star_\infty}$ on the annulus
$\{4 C_\infty \sigma^\infty/\mu_0^\infty \leq \norm{\theta -
\theta^\star_\infty} \leq r_\infty\}$.
\end{proposition}

\begin{proof}
Write $\theta = \theta^\star_\infty + x$ and Taylor-expand each operator.
Setting $p_e := P_e(\theta^\star_\infty)$ and $q_\infty := Q(\theta^\star_\infty)$:
\begin{align*}
Q(P_e(\theta)) &= Q(p_e) + DQ(p_e)\,B_e^\infty\, x + O(\norm{x}^2),\\
P_e(Q(\theta)) &= P_e(q_\infty) + DP_e(q_\infty)\,A_\infty\, x + O(\norm{x}^2).
\end{align*}
Decomposing $DQ(p_e) = A_\infty + (DQ(p_e) - A_\infty)$ and similarly for
$DP_e(q_\infty)$, and using uniform $C^2$ bounds, the cross-terms involving
$p_e - \theta^\star_\infty$ and $q_\infty - \theta^\star_\infty$ are
$O(\sigma^\infty)$ in expectation. Taking expectations and applying Jensen:
\[
  \mathbb{E}_e[\Ogap(\theta^\star_\infty + x; e)]
  \;\geq\; \norm{\bar\Sigma_\infty\, x} - C_\infty \sigma^\infty - R\norm{x}^2,
\]
for constants $C_\infty, R > 0$ depending on the smoothness bounds. Since
$\norm{\bar\Sigma_\infty x} \geq \mu_0^\infty \norm{x}$, choosing
$r_\infty := \mu_0^\infty/(2R)$ and restricting to $\norm{x} \leq r_\infty$
gives the bound. The annular reformulation follows by requiring the linear
term to dominate the noise term: $(\mu_0^\infty/2)\norm{x} \geq 2C_\infty
\sigma^\infty$. \qedhere
\end{proof}

In the noiseless regime, $\theta^\star_\infty = \theta^\star$,
$\sigma^\infty = 0$, the noise term vanishes, and
Proposition~\ref{prop:noisy_nondeg} reduces to
Proposition~\ref{prop:local_nondeg}. In the noisy regime, the natural
reference point is $\theta^\star_\infty$ rather than $\theta^\star$, and the
local order-gap lower bound is informative on the annulus where the linear
signal in $x$ exceeds the noise floor of order $\sigma^\infty$. This is the
formal content of the ``expected-update interpretation'' invoked in the
domain verifications of Section~\ref{sec:theory_domains}: the Jacobian rank
condition is checked at $\theta^\star_\infty$, and the sensitivity bound it
yields is valid on the annulus described above.

The stopping theorems of Section~\ref{sec:theory_stopping} are stated below
with respect to the consolidation fixed point $\theta^\star$. In the noisy
regime, the appropriate reference point is the effective equilibrium
$\theta^\star_\infty$ rather than $\theta^\star$;
Corollary~\ref{cor:stop_eff_eq} formalises this, restating the stopping
guarantees of Theorems~\ref{thm:stopping_det}--\ref{thm:stopping_hp} at
$\theta^\star_\infty$ with the appropriate substitutions. We retain the
$\theta^\star$-formulation in the main statements because it is the
structurally cleaner case and reduces to $\theta^\star_\infty$ in the
noiseless limit.

\subsection{Contraction of the Order-Gap}
\label{sec:theory_contraction}

\begin{theorem}[Contraction of the order-gap]
\label{thm:contraction}
Assume \ref{ass:contract}--\ref{ass:equilibrium} and set $\gamma := \rho L$.
Let $u_t := \mathbb{E}\bigl[\norm{\theta_t - \theta^\star}\bigr]$. Then:
\begin{enumerate}
\item[(i)] \textnormal{(Trajectory contraction, expectation.)} For $\gamma < 1$,
\begin{equation}
\label{eq:traj_bound}
u_t \;\leq\; \gamma^t u_0 + \rho\sigma \cdot \frac{1 - \gamma^t}{1 - \gamma}.
\end{equation}
\item[(ii)] \textnormal{(Order-gap bound, expectation.)}
\begin{equation}
\label{eq:ogap_bound}
\mathbb{E}\bigl[\Ogap(\theta_t; e_t)\bigr]
\;\leq\; 2\gamma\, u_t \;+\; (1 + \rho)\sigma.
\end{equation}
\item[(iii)] \textnormal{(Almost-sure geometric decay under $\sigma = 0$.)}
If $\sigma = 0$ and $\gamma < 1$, then a.s.,
\begin{equation}
\label{eq:geom_decay_as}
\norm{\theta_t - \theta^\star} \leq \gamma^t \norm{\theta_0 - \theta^\star},
\quad
\Ogap(\theta_t; e_t) \leq 2\gamma^{t+1} \norm{\theta_0 - \theta^\star}.
\end{equation}
\end{enumerate}
\end{theorem}

\begin{proof}
Write $W_t := \norm{P_{e_t}(\theta^\star) - \theta^\star}$. Using
$\theta^\star = Q(\theta^\star)$ and
Assumptions~\ref{ass:contract}--\ref{ass:lipschitz},
\begin{equation}
\label{eq:pathwise_contract}
\norm{\theta_{t+1} - \theta^\star}
\leq \rho L \norm{\theta_t - \theta^\star} + \rho W_t.
\end{equation}
Taking expectations gives (i). For (ii), a three-term triangle decomposition
(adding and subtracting $Q(P_e(\theta^\star))$ and $P_e(\theta^\star)$) gives:
\[
  \Ogap(\theta; e) \leq 2\rho L \norm{\theta - \theta^\star} + (1+\rho) W_e.
\]
Taking expectation yields \eqref{eq:ogap_bound}. Under $\sigma = 0$,
$W_t = 0$ a.s., so \eqref{eq:pathwise_contract} is a pathwise contraction
and \eqref{eq:geom_decay_as} follows by iteration and substitution. \qedhere
\end{proof}

\begin{remark}[Joint stability]
\label{rem:joint_stability}
The condition $\gamma = \rho L < 1$ couples the two operators: the
contraction of consolidation must outpace any expansion introduced by $P_e$.
This is the natural regime of interest; $\rho < 1$ alone merely says the
learner commits to its current best estimate, and $L$ can exceed~$1$ in
domains where $P_e$ amplifies evidence.
\end{remark}

\begin{remark}[Polynomial rates under weaker contraction]
\label{rem:polynomial_rates}
Assumption~\ref{ass:contract} can be weakened. If $Q$ is only
\emph{contractive in expectation} (e.g., $Q$ averages the current state with a running
sample mean, as in standard bandit updates), then $u_t$ decays at rate
$O(1/\sqrt{t})$ rather than geometrically, and $\mathbb{E}[\Ogap(\theta_t;
e_t)]$ inherits this rate. The $O(1/n)$ bounds familiar from stochastic
approximation fall out of this weaker regime. Domain applications that use
a non-contractive $Q$ (such as the bandit sharpening map of
Section~\ref{sec:theory_domains}) operate in this weaker regime.
\end{remark}

\begin{lemma}[Trajectory containment]
\label{lem:containment}
Under Assumptions~\ref{ass:contract}--\ref{ass:equilibrium} with
$W_e \leq M$ a.s.\ and $\gamma = \rho L < 1$, let
$R_0 := \norm{\theta_0 - \theta^\star}$. For every $r > \rho M/(1-\gamma)$,
define
\[
  T_0(r) :=
  \begin{cases}
    0, & R_0 \leq r - \rho M/(1-\gamma),\\[4pt]
    \left\lceil
      \dfrac{\log\bigl(R_0\,/\,(r - \rho M/(1-\gamma))\bigr)}{\log(1/\gamma)}
    \right\rceil, & \text{otherwise.}
  \end{cases}
\]
Then $\norm{\theta_t - \theta^\star} \leq r$ almost surely for all
$t \geq T_0(r)$.
\end{lemma}

\begin{proof}
Iterating \eqref{eq:pathwise_contract} with $W_s \leq M$ a.s.:
\[
  \norm{\theta_t - \theta^\star}
  \;\leq\; \gamma^t R_0 + \rho \sum_{s=0}^{t-1} \gamma^{t-1-s} W_s
  \;\leq\; \gamma^t R_0 + \rho M \cdot \frac{1 - \gamma^t}{1 - \gamma}
  \;\leq\; \gamma^t R_0 + \frac{\rho M}{1 - \gamma}
\]
almost surely. The right-hand side is at most $r$ when
$\gamma^t R_0 \leq r - \rho M/(1-\gamma)$, which holds for all
$t \geq T_0(r)$. \qedhere
\end{proof}

The containment radius $\rho M/(1-\gamma)$ is the noise-floor quantity
identified in the trajectory contraction bound \eqref{eq:traj_bound}; the
lemma upgrades it from a bound on the expected distance to a deterministic
pathwise bound under bounded noise. The condition $r > \rho M/(1-\gamma)$ is
the precise requirement that Assumption~\ref{ass:nondeg}'s validity radius
exceed the asymptotic noise ball radius. Combined with the noiseless
geometric containment $\norm{\theta_t - \theta^\star} \leq \gamma^t R_0$ from
Theorem~\ref{thm:contraction}(iii), the lemma certifies that the local
order-gap sensitivity assumption applies along the trajectory after a finite
deterministic transient $T_0(r)$ in both regimes.

\subsection{Operational Lower Bound}
\label{sec:theory_lowerbound}

\begin{proposition}[Order-gap lower bounds suboptimality]
\label{prop:lowerbound}
Under \ref{ass:contract}--\ref{ass:equilibrium}, assume additionally that
$\gamma = \rho L > 0$ (the case $\gamma = 0$ is excluded from the displayed
inverse bound; when $\gamma = 0$ the order-gap upper bound reduces to the
noise term $(1+\rho)\sigma$ alone). Then for every $\theta \in \X$:
\begin{equation}
\label{eq:dist_bound}
\norm{\theta - \theta^\star}
\geq \frac{1}{2\rho L}
      \bigl(\mathbb{E}_e[\Ogap(\theta; e)] - (1+\rho)\sigma\bigr)_+.
\end{equation}
If $\ell \circ \pi$ satisfies a \emph{quadratic growth} condition at
$\theta^\star$ with constant $m_{\mathrm{QG}} > 0$, meaning
$\ell(\pi(\theta)) - \ell(\pi(\theta^\star)) \geq
\tfrac{m_{\mathrm{QG}}}{2} \norm{\theta - \theta^\star}^2$, then
\begin{equation}
\label{eq:qg_bound}
\ell(\pi(\theta)) - \ell(\pi(\theta^\star))
\geq \frac{m_{\mathrm{QG}}}{8\rho^2 L^2}
      \bigl(\mathbb{E}_e[\Ogap(\theta; e)] - (1+\rho)\sigma\bigr)_+^2.
\end{equation}
\end{proposition}

\begin{proof}
Rearrange \eqref{eq:ogap_bound} to obtain \eqref{eq:dist_bound}. Squaring
and applying the quadratic growth lower bound gives \eqref{eq:qg_bound}. \qedhere
\end{proof}

\subsection{The Noise Floor}
\label{sec:theory_noisefloor}

When $\sigma > 0$, Theorem~\ref{thm:contraction} gives an \emph{upper bound}
on the asymptotic expected order-gap, which we call the \emph{noise-floor bound}:
\[
  \limsup_{t \to \infty}
  \mathbb{E}[\Ogap(\theta_t; e_t)]
  \leq \underbrace{(1 + \rho)\sigma}_{\varepsilon_\star^{\text{eq}}}
  + \underbrace{\frac{2\gamma\rho\sigma}{1 - \gamma}}_{\varepsilon_\star^{\text{traj}}}
  =: \varepsilon_\star.
\]
The actual asymptotic expected order-gap may be smaller; in special cases it
could be zero despite $\sigma > 0$. The decomposition nonetheless reveals
which part of the noise floor is reducible.

\textbf{Fixed-point residual:}
$\varepsilon_\star^{\text{eq}} = (1+\rho)\sigma$ is the contribution from
$\theta^\star$ itself. It is not controlled by trajectory contraction alone:
unless $\{P_e\}$ fixes $\theta^\star$, the equilibrium residual contributes
to the bound regardless of how well the trajectory contracts.
\textbf{Trajectory drift term:} $\varepsilon_\star^{\text{traj}} =
2\gamma\rho\sigma/(1-\gamma)$ is reducible: variance-reduced expansion
updates that shrink the effective $\sigma$ bring it down.

For the high-probability theorem we use the bounded-noise version of this floor.
Define the \emph{bounded-noise version of the floor}:
\begin{equation}
\label{eq:hp_floor}
\varepsilon_\star^{(M)} := (1+\rho)M + \frac{2\gamma\rho M}{1-\gamma}.
\end{equation}
This replaces $\sigma$ by $M$ (the a.s.\ bound on $W_e$) and serves as the
noise-floor quantity for the empirical windowed stopping theorem.

\begin{remark}[Practical implication]
\label{rem:practical_floor}
The sufficient stopping theorem requires the threshold $\varepsilon$ to exceed
the relevant noise-floor bound: $\varepsilon_\star$ in the noiseless expected-gap analysis,
or $\varepsilon_\star^{(M)}$ in the high-pro\-ba\-bil\-i\-ty theorem.
When the trajectory drift term dominates, variance reduction
pays dividends. When the fixed-point residual dominates, only modifying $P_e$
itself can help.
\end{remark}

\subsection{Stopping Guarantees}
\label{sec:theory_stopping}

We define two stopping rules, one for each regime.

\textbf{Expected-gap stopping rule} (noiseless/theoretical):
\begin{equation}
\label{eq:expected_window}
\bar\Ogap_{t,w} := \frac{1}{w} \sum_{s = t-w}^{t-1}
  \mathbb{E}_e\bigl[\Ogap(\theta_s; e)\bigr],
\qquad
\bar\tau_\varepsilon := \inf\bigl\{t \geq w :
  \bar\Ogap_{t,w} \leq \varepsilon \bigr\}.
\end{equation}

\textbf{Empirical windowed stopping rule} (operational):
\begin{equation}
\label{eq:realized_window}
\widehat{\Ogap}_{t,w} := \frac{1}{w} \sum_{s = t-w}^{t-1} \Ogap(\theta_s; e_s),
\qquad
\hat\tau_\varepsilon := \inf\bigl\{t \geq w :
  \widehat{\Ogap}_{t,w} \leq \varepsilon \bigr\}.
\end{equation}

Both windows run from $s = t-w$ to $s = t-1$. Since $\mathcal{F}_t =
\sigma(e_0,\ldots,e_{t-1})$, both $\bar\Ogap_{t,w}$ and $\widehat{\Ogap}_{t,w}$
are $\mathcal{F}_t$-measurable, so $\bar\tau_\varepsilon$ and
$\hat\tau_\varepsilon$ are stopping times with respect to $\{\mathcal{F}_t\}$.
Algorithm~\ref{alg:windowed_stop} gives the operational form of the
empirical rule, which is what an implementation actually runs:
the order-gap is computed pathwise from a single forward pass, the
windowed average is updated incrementally, and stopping triggers when
the average drops below $\varepsilon$.

\begin{algorithm}[H]
\caption{Empirical windowed stopping (one trajectory)}
\label{alg:windowed_stop}
\begin{algorithmic}[1]
\Require operators $Q$, $\{P_e\}_{e \in \E}$; sampler $\E$;
         initial state $\theta_0$; threshold $\varepsilon > 0$;
         window size $w \geq 1$; max steps $T_{\max}$.
\Ensure stopping time $\hat\tau$ and final state $\theta_{\hat\tau}$.
\State buffer $\mathcal{B} \gets$ empty FIFO queue of capacity $w$
\For{$t = 0, 1, 2, \ldots, T_{\max}$}
    \State sample $e_t \sim \E$
    \State $\omega_t \gets \norm{Q(P_{e_t}(\theta_t)) - P_{e_t}(Q(\theta_t))}$
    \Comment{order-gap at step $t$}
    \State enqueue $\omega_t$ into $\mathcal{B}$
    \If{$|\mathcal{B}| = w$ and $\frac{1}{w}\sum_{\omega \in \mathcal{B}} \omega \leq \varepsilon$}
        \State \Return $(\hat\tau, \theta_{\hat\tau}) \gets (t+1, \theta_{t+1})$
    \EndIf
    \State $\theta_{t+1} \gets Q(P_{e_t}(\theta_t))$
    \Comment{advance state}
\EndFor
\State \Return $(T_{\max}, \theta_{T_{\max}})$
\Comment{budget exhausted}
\end{algorithmic}
\end{algorithm}

The cost per step is two operator evaluations beyond the single $Q \circ P_e$
the system would compute anyway: one $P_e \circ Q$ and one norm. The
buffer is $O(w)$ memory. Theorems~\ref{thm:stopping_det}
and~\ref{thm:stopping_hp} (with their preconditions) certify that
Algorithm~\ref{alg:windowed_stop} returns within $w + N_\varepsilon$ or
$T_0$ steps respectively, with the corresponding suboptimality bounds.

\subsubsection*{The noiseless regime: $\sigma = 0$}

\begin{theorem}[Expected-gap stopping, $\sigma = 0$]
\label{thm:stopping_det}
Assume \ref{ass:contract}--\ref{ass:nondeg} with $\sigma = 0$, $\gamma < 1$,
and $R_0 := \norm{\theta_0 - \theta^\star}$. If $\gamma = 0$ or $R_0 = 0$
then $\bar\tau_\varepsilon = w$. Otherwise, for every $\varepsilon > 0$,
$w \geq 1$, define
\[
N_\varepsilon :=
\begin{cases}
0, & 2\gamma R_0 \leq \varepsilon,\\[4pt]
\left\lceil
  \dfrac{\log(2\gamma R_0/\varepsilon)}{\log(1/\gamma)}
\right\rceil, & 2\gamma R_0 > \varepsilon.
\end{cases}
\]
\begin{enumerate}
\item[(i)] \textnormal{(Stopping-time bound.)}
\begin{equation}
\label{eq:stop_time_det}
\bar\tau_\varepsilon \;\leq\; w + N_\varepsilon \qquad \text{a.s.}
\end{equation}
In particular, when $2\gamma R_0 \leq \varepsilon$ we have $N_\varepsilon = 0$
and the system stops at $\bar\tau_\varepsilon = w$.

\item[(ii)] \textnormal{(Monotone distance decrease.)}
$\norm{\theta_s - \theta^\star} \geq \norm{\theta_{s+1} - \theta^\star}$
a.s.\ for every $s$.

\item[(iii)] \textnormal{(Suboptimality at stopping.)}
If, in addition, $\varepsilon \leq 2\gamma r$ where $r$ is the validity radius
of Assumption~\ref{ass:nondeg}, then
\[
  \norm{\theta_{\bar\tau_\varepsilon} - \theta^\star}
  \;\leq\; \frac{\varepsilon}{\mu} \qquad \text{a.s.}
\]
The condition $\varepsilon \leq 2\gamma r$ ensures the stopping window
$\mathcal{W}$ lies inside the validity ball $B_r(\theta^\star)$ before
stopping triggers; for $\varepsilon > 2\gamma r$ the stopping rule may
terminate while the trajectory is still outside $B_r(\theta^\star)$ and
Assumption~\ref{ass:nondeg} does not apply.
\end{enumerate}
\end{theorem}

\begin{proof}
\emph{(i).}
Under $\sigma = 0$, from Theorem~\ref{thm:contraction}(iii),
$\mathbb{E}_e[\Ogap(\theta_s; e)] \leq 2\gamma\norm{\theta_s - \theta^\star}
\leq 2\gamma^{s+1} R_0$ a.s. Averaging over the window $[t-w, t-1]$:
\[
  \bar\Ogap_{t,w}
  \leq \frac{1}{w}\sum_{s=t-w}^{t-1} 2\gamma^{s+1} R_0
  \leq 2\gamma^{t-w+1} R_0 \qquad \text{a.s.}
\]
This is $\leq \varepsilon$ once $2\gamma^{t-w+1} R_0 \leq \varepsilon$, i.e.,
$t \geq w - 1 + \log(2\gamma R_0/\varepsilon)/\log(1/\gamma)$ when
$2\gamma R_0 > \varepsilon$, and immediately at $t = w$ otherwise.
The piecewise definition of $N_\varepsilon$ handles both cases.

\emph{(ii).}
From \eqref{eq:pathwise_contract} with $W_t = 0$ a.s.:
$\norm{\theta_{s+1} - \theta^\star} \leq \gamma \norm{\theta_s - \theta^\star}
\leq \norm{\theta_s - \theta^\star}$.

\emph{(iii).}
\textbf{Trajectory containment.}
By Theorem~\ref{thm:contraction}(iii), $\norm{\theta_s - \theta^\star} \leq
\gamma^s R_0$ a.s.\ for every $s$, so the trajectory enters the validity
ball $B_r(\theta^\star)$ of Assumption~\ref{ass:nondeg} at the deterministic
time $T_0^{\mathrm{det}}(r) := \lceil \log(R_0/r)/\log(1/\gamma) \rceil$.
Under the precondition $\varepsilon \leq 2\gamma r$,
\[
  N_\varepsilon
  = \left\lceil \frac{\log(2\gamma R_0/\varepsilon)}{\log(1/\gamma)} \right\rceil
  \;\geq\; \left\lceil \frac{\log(R_0/r)}{\log(1/\gamma)} \right\rceil
  = T_0^{\mathrm{det}}(r),
\]
so the stopping window $\mathcal{W} := [\bar\tau_\varepsilon - w,
\bar\tau_\varepsilon - 1]$ lies in $\{s \geq T_0^{\mathrm{det}}(r)\}$ and
$\theta_s \in B_r(\theta^\star)$ for every $s \in \mathcal{W}$.

\textbf{Conclusion.}
By (ii), for every $s \in \mathcal{W}$,
$\norm{\theta_s - \theta^\star} \geq \norm{\theta_{\bar\tau_\varepsilon} -
\theta^\star}$ a.s. Applying Assumption~\ref{ass:nondeg} termwise (valid
since $\theta_s \in B_r(\theta^\star)$ for $s \in \mathcal{W}$):
\[
  \varepsilon \geq \bar\Ogap_{\bar\tau_\varepsilon, w}
  = \frac{1}{w}\sum_{s \in \mathcal{W}} \mathbb{E}_e[\Ogap(\theta_s; e)]
  \geq \frac{\mu}{w}\sum_{s \in \mathcal{W}} \norm{\theta_s - \theta^\star}
  \geq \mu\, \norm{\theta_{\bar\tau_\varepsilon} - \theta^\star}.
\]
Rearranging gives the claim. \qedhere
\end{proof}

\subsubsection*{The noisy regime: $\sigma > 0$ with bounded noise}

\begin{theorem}[Empirical windowed stopping, $\sigma > 0$]
\label{thm:stopping_hp}
Assume \ref{ass:contract}--\ref{ass:nondeg} with $W_e \leq M$ a.s. Let
$R_0 := \norm{\theta_0 - \theta^\star}$, $\gamma = \rho L$, and assume
$0 < \gamma < 1$ for the displayed logarithmic bound; the degenerate
cases $\gamma = 0$ or $R_0 = 0$ are handled by setting $N_{\varepsilon,M}=0$,
giving $T_0 = w$ and $\hat\tau_\varepsilon \leq w$ on $\mathcal{G}$. Let
$\varepsilon > \varepsilon_\star^{(M)}$ and $\delta \in (0,1)$. Define
the a.s.\ order-gap envelope
\[
  K := 2\rho L \!\left(R_0 + \frac{\rho M}{1 - \gamma}\right) + (1+\rho)M,
\]
and
\[
N_{\varepsilon,M} :=
\begin{cases}
0, & 4R_0 \leq \varepsilon - \varepsilon_\star^{(M)},\\[4pt]
\left\lceil
  \dfrac{\log\!\bigl(4R_0/(\varepsilon-\varepsilon_\star^{(M)})\bigr)}
       {\log(1/\gamma)}
\right\rceil, & 4R_0 > \varepsilon - \varepsilon_\star^{(M)},
\end{cases}
\]
so that $T_0 := w + N_{\varepsilon,M}$.
Let $r$ be the validity radius of Assumption~\ref{ass:nondeg}, with
$r > \rho M/(1-\gamma)$.
If
\begin{equation}
\label{eq:w_condition}
w \;\geq\;
\frac{8 K^2}{(\varepsilon - \varepsilon_\star^{(M)})^2}
\log\!\left(\frac{2 T_0}{\delta}\right),
\end{equation}
then there exists a concentration event $\mathcal{G}$ with
$\mathbb{P}(\mathcal{G}) \geq 1 - \delta$ (defined precisely in the proof)
on which:
\begin{enumerate}
\item[(i)] \textnormal{(Stopping-time bound.)} $\hat\tau_\varepsilon \leq T_0$.

\item[(ii)] \textnormal{(Windowed-average bound, a.s.\ on $\mathcal{G}$.)}
If, in addition,
\begin{equation}
\label{eq:noisy_compat}
  r \;\geq\; \frac{\rho M}{1-\gamma} \;+\; \frac{\varepsilon - \varepsilon_\star^{(M)}}{4},
\end{equation}
then
\[
  \frac{1}{w}\sum_{s \in \mathcal{W}} \norm{\theta_s - \theta^\star}
  \;\leq\; \frac{\varepsilon + \eta}{\mu},
\]
where $\mathcal{W} := [\hat\tau_\varepsilon - w, \hat\tau_\varepsilon - 1]$
and $\eta := K\sqrt{2 \log(2 T_0 / \delta) / w}$.

\item[(iii)] \textnormal{(Pathwise endpoint bound, on $\mathcal{G}$.)} Under
condition~\eqref{eq:noisy_compat},
\[
  \norm{\theta_{\hat\tau_\varepsilon} - \theta^\star}
  \;\leq\; \frac{\varepsilon + \eta}{\mu} + \frac{\rho M}{1 - \gamma}
  \qquad \text{a.s.\ on } \mathcal{G}.
\]
\end{enumerate}
\end{theorem}

\begin{proof}
\emph{Pathwise order-gap envelope.}
From \eqref{eq:pathwise_contract} with $W_s \leq M$ a.s.:
$\norm{\theta_t - \theta^\star} \leq \gamma^t R_0 + \rho M/(1-\gamma) \leq
R_0 + \rho M/(1-\gamma)$ a.s. Combined with the triangle bound from
Theorem~\ref{thm:contraction}:
\[
  \Ogap(\theta_t; e_t) \leq 2\rho L\!\left(R_0 + \tfrac{\rho M}{1-\gamma}\right) + (1+\rho) M
  = K \quad\text{a.s.}
\]

\emph{Pathwise bounding of $g_s$.}
Define $g_s := \mathbb{E}_e[\Ogap(\theta_s; e)]$, the conditional mean.
Since $e_s$ is independent of $\mathcal{F}_s$, we have
$g_s = \mathbb{E}[\Ogap(\theta_s; e_s) \mid \mathcal{F}_s]$.
Using the triangle bound with $W_e \leq M$ and the pathwise distance bound:
\begin{equation}
\label{eq:gs_pathwise}
g_s \;\leq\; 2\gamma \norm{\theta_s - \theta^\star} + (1+\rho)M
\;\leq\; 2\gamma(\gamma^s R_0 + \tfrac{\rho M}{1-\gamma}) + (1+\rho)M
= 2\gamma^{s+1} R_0 + \varepsilon_\star^{(M)} \quad\text{a.s.}
\end{equation}

\emph{Martingale decomposition.}
The residuals $\Xi_s := \Ogap(\theta_s; e_s) - g_s$ form a martingale
difference sequence with respect to $\{\mathcal{F}_{s+1}\}$, with
$|\Xi_s| \leq K$ a.s.

\emph{Two-sided concentration event.}
By Azuma--Hoeffding applied at each $t \in [w, T_0]$ and a union bound,
define $\mathcal{G}$ as the event that for \emph{every} $t \in [w, T_0]$,
\begin{equation}
\label{eq:two_sided}
\left|\widehat{\Ogap}_{t,w} - \frac{1}{w}\sum_{s=t-w}^{t-1} g_s\right|
\;\leq\; \eta.
\end{equation}
Setting $\eta = K\sqrt{2 \log(2 T_0 / \delta) / w}$ gives
$\mathbb{P}(\mathcal{G}^c) \leq \delta$.

\emph{(i) Stopping by $T_0$.}
Under $\mathcal{G}$ (upper half of \eqref{eq:two_sided}) and \eqref{eq:gs_pathwise}:
$\widehat{\Ogap}_{t,w} \leq 2\gamma^{t-w+1} R_0 + \varepsilon_\star^{(M)} + \eta$.
By the choice of $T_0$, at $t = T_0$:
$2\gamma^{T_0-w+1} R_0 \leq (\varepsilon - \varepsilon_\star^{(M)})/2$.
By \eqref{eq:w_condition}: $\eta \leq (\varepsilon - \varepsilon_\star^{(M)})/2$.
Hence $\widehat{\Ogap}_{T_0, w} \leq \varepsilon$, so $\hat\tau_\varepsilon \leq T_0$.

\emph{(ii) Windowed-average suboptimality.}
At stopping, $\widehat{\Ogap}_{\hat\tau, w} \leq \varepsilon$. Under $\mathcal{G}$
(both sides of \eqref{eq:two_sided}):
$\frac{1}{w}\sum_{s \in \mathcal{W}} g_s \leq \varepsilon + \eta$.
\textbf{Trajectory containment.}
By Lemma~\ref{lem:containment} with the validity radius $r$ of
Assumption~\ref{ass:nondeg}, the trajectory satisfies $\norm{\theta_s -
\theta^\star} \leq r$ a.s.\ for $s \geq T_0(r)$, where
\[
  T_0(r) =
  \left\lceil
  \frac{\log\bigl(R_0/(r - \rho M/(1-\gamma))\bigr)}{\log(1/\gamma)}
  \right\rceil
\]
when $R_0 > r - \rho M/(1-\gamma)$, and $T_0(r) = 0$ otherwise. Under the
compatibility condition~\eqref{eq:noisy_compat},
\[
  N_{\varepsilon, M}
  = \left\lceil
    \frac{\log\!\bigl(4R_0/(\varepsilon - \varepsilon_\star^{(M)})\bigr)}
         {\log(1/\gamma)}
    \right\rceil
  \;\geq\;
  \left\lceil
    \frac{\log\bigl(R_0/(r - \rho M/(1-\gamma))\bigr)}{\log(1/\gamma)}
  \right\rceil
  = T_0(r),
\]
so the stopping window $\mathcal{W}$ lies in $\{s \geq T_0(r)\}$ and
$\theta_s \in B_r(\theta^\star)$ for every $s \in \mathcal{W}$.
Assumption~\ref{ass:nondeg} applies termwise:
$g_s \geq \mu\norm{\theta_s - \theta^\star}$. Averaging over the window
gives (ii).

\emph{(iii) Endpoint bound.}
By the pigeonhole principle applied to (ii), there exists
$s^\dagger \in \mathcal{W}$ with
$\norm{\theta_{s^\dagger} - \theta^\star} \leq (\varepsilon + \eta)/\mu$.
Iterating \eqref{eq:pathwise_contract} from $s^\dagger$ to $\hat\tau_\varepsilon$
and using $\gamma^k \leq 1$ and $W_j \leq M$:
\[
  \norm{\theta_{\hat\tau} - \theta^\star}
  \leq \norm{\theta_{s^\dagger} - \theta^\star} + \frac{\rho M}{1 - \gamma}
  \leq \frac{\varepsilon + \eta}{\mu} + \frac{\rho M}{1 - \gamma}
  \qquad \text{a.s.\ on } \mathcal{G}. \qedhere
\]
\end{proof}

\begin{corollary}[Operational stopping and excess loss]
\label{cor:op_stopping}
Under the hypotheses of Theorem~\ref{thm:stopping_det}
or~\ref{thm:stopping_hp}, suppose additionally that $\ell \circ \pi$ satisfies
a \emph{local smoothness} condition at $\theta^\star$ with constant
$M_{\mathrm{sm}} > 0$:
$\ell(\pi(\theta)) - \ell(\pi(\theta^\star)) \leq \tfrac{M_{\mathrm{sm}}}{2}
\norm{\theta - \theta^\star}^2$ for $\theta$ in a neighbourhood of
$\theta^\star$ of radius $r_{\mathrm{sm}} > 0$. Assume further that the
endpoint distance bounds derived below lie within that neighbourhood (or that the
quadratic upper bound holds globally on the relevant sublevel set).
This is an \emph{upper} growth condition distinct from the quadratic growth
lower bound $m_{\mathrm{QG}}$ of Proposition~\ref{prop:lowerbound}.
Then:

\textbf{Noiseless:}
$\displaystyle
\ell\bigl(\pi(\theta_{\bar\tau_\varepsilon})\bigr)
- \ell\bigl(\pi(\theta^\star)\bigr)
\leq \frac{M_{\mathrm{sm}}}{2} \left(\frac{\varepsilon}{\mu}\right)^2.$

\textbf{Noisy high-probability:} On the concentration event $\mathcal{G}$ (probability $\geq 1-\delta$),
\[
\ell\bigl(\pi(\theta_{\hat\tau_\varepsilon})\bigr)
- \ell\bigl(\pi(\theta^\star)\bigr)
\leq \frac{M_{\mathrm{sm}}}{2}
\left(\frac{\varepsilon + \eta}{\mu} + \frac{\rho M}{1 - \gamma}\right)^2.
\]
\end{corollary}

\begin{proof}
Apply the local smoothness bound to the endpoint distance bounds of
Theorem~\ref{thm:stopping_det}(iii) and Theorem~\ref{thm:stopping_hp}(iii)
respectively. The noisy bound retains the $\rho M/(1-\gamma)$ drift term; it
cannot be replaced by $\rho\sigma/(1-\gamma)$ because the stopping time
$\hat\tau_\varepsilon$ is selected adaptively and the selected $W_j$ values
are not an unconditional sample. \qedhere
\end{proof}

The condition \eqref{eq:w_condition} is self-consistent: $T_0$ depends on
$w$ and $\log(2T_0/\delta)$ appears in the bound. One may resolve this by
substituting a crude upper bound for $T_0$ (e.g., replace $\log(2T_0/\delta)$
with $\log(2(T_0^{\mathrm{ub}})/\delta)$ where $T_0^{\mathrm{ub}}$ is any
upper bound on $T_0$), or by choosing any $w$ exceeding the right-hand side
of \eqref{eq:w_condition} with $T_0$ evaluated at that $w$.

\begin{remark}[Reading the bound]
\label{rem:hp_reading}
The window size $w$ trades off three quantities: concentration width
$\eta \propto 1/\sqrt{w}$, the confidence $\delta$, and the usable stopping
threshold $\varepsilon$. The minimal viable $w$ scales as
$K^2/(\varepsilon - \varepsilon_\star^{(M)})^2$: the price of pushing the
stopping threshold close to the noise-floor bound is a quadratically larger
window.
\end{remark}

\begin{corollary}[Stopping at the effective equilibrium]
\label{cor:stop_eff_eq}
Assume \ref{ass:contract}--\ref{ass:lipschitz} with $\gamma = \rho L < 1$,
let $\theta^\star_\infty$ be the effective equilibrium of
Proposition~\ref{prop:effective_eq}, and define
$W_e^\infty := \norm{P_e(\theta^\star_\infty) - \theta^\star_\infty}$ with
$\mathbb{E}[W_e^\infty] \leq \sigma^\infty$ and $W_e^\infty \leq M^\infty$
a.s. Suppose Assumption~\ref{ass:nondeg} holds at $\theta^\star_\infty$ with
constant $\mu_\infty > 0$ and validity radius $r_\infty > \rho
M^\infty/(1-\gamma)$ (typically established via
Proposition~\ref{prop:noisy_nondeg} with the compatibility condition
$\mu_0^\infty > 2 R \rho M^\infty/(1-\gamma)$ of
Remark~\ref{rem:radius_compat}). Define
\[
  R_0^\infty := \norm{\theta_0 - \theta^\star_\infty}, \qquad
  \varepsilon_\star^{(M^\infty)} := (1+\rho)M^\infty + \frac{2\gamma\rho M^\infty}{1-\gamma}.
\]
Then Theorems~\ref{thm:stopping_det}--\ref{thm:stopping_hp} hold verbatim
with the substitutions
\[
  \theta^\star \to \theta^\star_\infty, \quad
  R_0 \to R_0^\infty, \quad
  \mu \to \mu_\infty, \quad
  M \to M^\infty, \quad
  \varepsilon_\star^{(M)} \to \varepsilon_\star^{(M^\infty)}, \quad
  r \to r_\infty,
\]
including the trajectory-containment preconditions: part (iii) of
Theorem~\ref{thm:stopping_det} requires $\varepsilon \leq 2\gamma
r_\infty$, and parts (ii)--(iii) of Theorem~\ref{thm:stopping_hp} require
$r_\infty \geq \rho M^\infty/(1-\gamma) + (\varepsilon -
\varepsilon_\star^{(M^\infty)})/4$. All distance bounds are with respect to
$\theta^\star_\infty$.
\end{corollary}

\begin{proof}
Centering coordinates at $\theta^\star_\infty$ rather than $\theta^\star$,
the path-wise contraction inequality \eqref{eq:pathwise_contract} becomes
\[
  \norm{\theta_{t+1} - \theta^\star_\infty}
  \leq \rho L \norm{\theta_t - \theta^\star_\infty} + \rho W_t^\infty,
\]
with $W_t^\infty := \norm{P_{e_t}(\theta^\star_\infty) -
\theta^\star_\infty}$ playing the role of $W_t$. This follows from
$\theta^\star_\infty = T(\theta^\star_\infty) =
\mathbb{E}_e[Q(P_e(\theta^\star_\infty))]$ together with
Assumptions~\ref{ass:contract}--\ref{ass:lipschitz}; the same triangle
decomposition that yielded \eqref{eq:pathwise_contract} at $\theta^\star$
goes through at $\theta^\star_\infty$ because $Q$ is contractive globally
(not merely at its fixed point) and the centering only relabels the
zeroth-order term in the Taylor expansion. With this substitution, every
step of the proofs of Theorems~\ref{thm:contraction},
\ref{thm:stopping_det}, and \ref{thm:stopping_hp} carries through verbatim,
yielding the stated bounds with $\theta^\star \to \theta^\star_\infty$, $W_e
\to W_e^\infty$, and $\varepsilon_\star^{(M)} \to
\varepsilon_\star^{(M^\infty)}$. \qedhere
\end{proof}

In stochastic domains where $\sigma > 0$ (bandits, off-policy RL),
Corollary~\ref{cor:stop_eff_eq} is the operative form of the
stopping guarantee: the trajectory does not converge to $\theta^\star$ but
to $\theta^\star_\infty$, and the stopping bounds are correspondingly
referenced to $\theta^\star_\infty$. In the noiseless limit ($\sigma = 0$,
hence $\sigma^\infty = M^\infty = 0$ and $\theta^\star_\infty =
\theta^\star$), the corollary reduces to Theorem~\ref{thm:stopping_det}.

\subsection{Extension to State-Dependent Sampling}
\label{sec:state_dependent}

The theorems above assume events $e_t$ are drawn i.i.d.\ from a fixed
distribution $P$, independent of $\theta_t$. For domains where the sampling
kernel depends on the current knowledge state — bandits with adaptive arm
selection, off-policy reinforcement learning with state-dependent behaviour
policies — this independence
fails. We extend the contraction and stopping theorems to the state-dependent
case under uniform conditional moment bounds. The crucial observation is
that the martingale structure underlying the concentration argument
\emph{survives} the move from i.i.d.\ to state-dependent sampling, so no
mixing-time apparatus is required.

\paragraph{Setup.}
Let $\{(\theta_t, e_t)\}_{t \geq 0}$ be a Markov process on $\X \times \E$
in which $e_t$ is drawn from a state-dependent kernel $P(\cdot \mid \theta_t)$.
Assume the kernel is $\mathcal{F}_t$-measurable, i.e., $\theta_t$ is
$\mathcal{F}_t$-measurable and $e_t$ is drawn independently of
$\mathcal{F}_t$ given $\theta_t$. The dynamics
$\theta_{t+1} = Q \circ P_{e_t}(\theta_t)$ continue to define an
$\mathcal{F}_t$-adapted Markov chain on $\X$.

\begin{assumption}[Uniform conditional moment bounds]
\label{ass:cond_moments}
Let $W_e := \norm{P_e(\theta^\star) - \theta^\star}$. There exist
$\sigma \geq 0$ and $M \geq \sigma$ such that, for every $\theta \in \X$,
\[
  \int_\E W_e \, P(de \mid \theta) \;\leq\; \sigma,
  \qquad
  \mathrm{ess\,sup}_{e \sim P(\cdot \mid \theta)}\, W_e \;\leq\; M.
\]
In the i.i.d.\ case, these reduce to the marginal expectations of
Assumption~\ref{ass:equilibrium} and the analysis of
Sections~\ref{sec:theory_contraction}--\ref{sec:theory_stopping} is
recovered.
\end{assumption}

Assumption~\ref{ass:nondeg} (local order-gap sensitivity) is similarly
interpreted at the conditional level: for $\theta \in B_r(\theta^\star)$
(or $B_{r_\infty}(\theta^\star_\infty)$ in the noisy regime),
\[
  \int_\E \Ogap(\theta; e) \, P(de \mid \theta) \;\geq\; \mu \norm{\theta - \theta^\star},
\]
which agrees with the i.i.d.\ form when the conditional kernel is the fixed
marginal $P$.

\begin{theorem}[Contraction under state-dependent sampling]
\label{thm:contraction_cond}
Under Assumptions~\ref{ass:contract}--\ref{ass:lipschitz} and
Assumption~\ref{ass:cond_moments}, the conclusions of
Theorem~\ref{thm:contraction} hold verbatim with the contraction rate
$\gamma = \rho L$ unchanged.
\end{theorem}

\begin{proof}
The path-wise contraction inequality \eqref{eq:pathwise_contract},
\[
  \norm{\theta_{t+1} - \theta^\star}
  \leq \rho L \norm{\theta_t - \theta^\star} + \rho W_t,
\]
holds path-wise with $W_t := \norm{P_{e_t}(\theta^\star) - \theta^\star}$,
independently of how $e_t$ is sampled, because the inequality follows from
Assumptions~\ref{ass:contract}--\ref{ass:lipschitz} pointwise in $e$. Taking
expectations,
\[
  \mathbb{E}[W_t]
  = \mathbb{E}\!\left[\,\int_\E W_e \, P(de \mid \theta_t)\right]
  \leq \sigma,
\]
by Assumption~\ref{ass:cond_moments}. Iterating gives
$u_{t+1} \leq \gamma u_t + \rho\sigma$, hence
$u_t \leq \gamma^t u_0 + \rho\sigma(1 - \gamma^t)/(1-\gamma)$. The pathwise
bound $W_t \leq M$ a.s.\ holds analogously, yielding
$\norm{\theta_t - \theta^\star} \leq \gamma^t R_0 + \rho M /(1-\gamma)$ a.s.
The triangle bound on $\Ogap$ in Theorem~\ref{thm:contraction}(ii) is
unchanged because it follows pointwise in $e$. \qedhere
\end{proof}

\begin{theorem}[Empirical stopping under state-dependent sampling]
\label{thm:stopping_cond}
Under Assumptions~\ref{ass:contract}--\ref{ass:nondeg} (with
Assumption~\ref{ass:cond_moments} replacing
Assumption~\ref{ass:equilibrium} and Assumption~\ref{ass:nondeg}
interpreted at the conditional level), the conclusions of
Theorem~\ref{thm:stopping_hp} hold verbatim, with $g_s := \int_\E
\Ogap(\theta_s; e) P(de \mid \theta_s) = \mathbb{E}[\Ogap(\theta_s; e_s) \mid
\mathcal{F}_s]$ and the same envelope $K$, concentration term $\eta$, and
window-size requirement.
\end{theorem}

\begin{proof}
We verify that each step of the proof of Theorem~\ref{thm:stopping_hp}
carries through.

\emph{Pathwise envelope.} The bound $\Ogap(\theta_t; e_t) \leq K$ a.s.\
follows from Theorem~\ref{thm:contraction_cond} and the triangle bound,
exactly as in the i.i.d.\ case.

\emph{Pathwise bounding of $g_s$.} Because $\theta_s$ is
$\mathcal{F}_s$-measurable and $e_s$ is drawn from $P(\cdot \mid \theta_s)$
independently of $\mathcal{F}_s$ given $\theta_s$,
\[
  g_s
  := \int_\E \Ogap(\theta_s; e)\, P(de \mid \theta_s)
  = \mathbb{E}\bigl[\Ogap(\theta_s; e_s) \mid \mathcal{F}_s\bigr],
\]
which is exactly the role $g_s$ plays in the proof of
Theorem~\ref{thm:stopping_hp}. The triangle bound
$g_s \leq 2\gamma\norm{\theta_s - \theta^\star} + (1+\rho)M$ continues to
hold, now under the uniform conditional bound on $W_e$.

\emph{Martingale decomposition.} The residuals
$\Xi_s := \Ogap(\theta_s; e_s) - g_s$ form a martingale difference
sequence with respect to $\{\mathcal{F}_{s+1}\}$, with $|\Xi_s| \leq K$
a.s. The martingale property holds because, by construction,
$\mathbb{E}[\Ogap(\theta_s; e_s) \mid \mathcal{F}_s] = g_s$ — and this
identity is the only property of the i.i.d.\ structure used in the proof of
Theorem~\ref{thm:stopping_hp}. State-dependent sampling preserves this
identity: it merely replaces a fixed marginal expectation by a
state-conditional one.

\emph{Azuma--Hoeffding concentration.}
Because $\{\Xi_s\}$ is a bounded martingale difference sequence (whether
$e_s$ is drawn i.i.d.\ or from a state-dependent kernel),
Azuma--Hoeffding applied to $\sum_{s=t-w}^{t-1} \Xi_s$ yields the same
two-sided concentration event $\mathcal{G}$ with
$\eta = K\sqrt{2\log(2T_0/\delta)/w}$, by the same union bound. No
mixing-time correction enters because the relevant concentration tool is
\emph{martingale} concentration, not Markov-chain concentration: we are
controlling $\widehat{\Ogap}_{t,w} - (1/w)\sum_s g_s$, not
$(1/w)\sum_s \Ogap(\theta_s; e_s) - \mathbb{E}_{\pi^\star}[\Ogap]$.

\emph{Trajectory containment, stopping bound, endpoint bound.}
These steps depend only on the pathwise contraction and on
Assumption~\ref{ass:nondeg} applied at conditional expectations, both of
which we have established. The arguments are unchanged. \qedhere
\end{proof}

\paragraph{Operational implication.}
The cost of state-dependent sampling, in this formulation, is borne by the
strengthening of the moment hypothesis: the marginal expectation in
Assumption~\ref{ass:equilibrium} is replaced by the supremum over
conditioning states in Assumption~\ref{ass:cond_moments}. Whenever this
strengthening holds — which is automatic when the operator family $\{P_e\}$
satisfies a uniform Lipschitz bound on a bounded state space, or when
exploration in bandits and RL is bounded below uniformly across states —
the i.i.d.\ stopping guarantees apply with no further modification. Mixing
of the joint chain is not required because the analysis controls
\emph{trajectory} averages against \emph{conditional} means, both of which
are well-defined for adapted processes. The corresponding restatement at
the effective equilibrium $\theta^\star_\infty$ follows by combining
Corollary~\ref{cor:stop_eff_eq} with Theorem~\ref{thm:stopping_cond}.

\subsection{Verifying the Theoretical Conditions in Three Domains}
\label{sec:theory_domains}

The bandit, RLM, and actor-critic cases receive explicit analysis of when
the order-gap reliably tracks convergence below. The SGD case
(Section~\ref{sec:sgd}) reduces to classical Robbins-Monro stochastic
approximation \citep{robbins1951stochastic,borkar2008stochastic} when $Q$
is trivial; the contraction theorem applies verbatim, and the order-gap
becomes a non-trivial diagnostic in adaptive variants (momentum, Adam,
proximal). The continual-learning case is a qualitative mapping of the
framework rather than a full theorem verification (Section~\ref{sec:domains}):
it shows how the $Q$/$P_e$ decomposition and order-gap control principle
map onto an existing algorithm family.

\textbf{Effective equilibrium and Jacobian reference point.} In stochastic
domains (bandits, RL), the raw sample-level expansion operator satisfies
$P_e(\theta^\star) \neq \theta^\star$ in general: a Gaussian reward draw or
a single TD transition perturbs $\theta^\star$ stochastically. The noiseless
hypothesis $\sigma = 0$ of Proposition~\ref{prop:local_nondeg} therefore does
not hold for the raw operators, and the Jacobian calculations below cannot be
performed at $\theta^\star$. The framework's formal handling of this case is
Proposition~\ref{prop:noisy_nondeg}: the Jacobians are evaluated at the
effective equilibrium $\theta^\star_\infty$ (the unique fixed point of
$T(\theta) = \mathbb{E}_e[Q(P_e(\theta))]$, established in
Proposition~\ref{prop:effective_eq}), and the order-gap sensitivity bound
holds on the annulus where the linear signal in the displacement
$\theta - \theta^\star_\infty$ exceeds the residual noise level
$\sigma^\infty$ at $\theta^\star_\infty$. The Jacobian calculations below
should be read in this sense: as commutator computations at
$\theta^\star_\infty$ rather than at $\theta^\star$. The first-moment or
Gramian conditions, when satisfied, establish the order-gap lower bound on
the relevant annulus around $\theta^\star_\infty$. The bandit verification
below shows the Gramian path explicitly: in the Bayesian setup with
posterior contraction and mean shrinkage, the first-moment commutator
vanishes at the unbiased equilibrium while the second-moment Gramian remains
non-zero under standard exploration coverage.

\subsubsection*{Bandits: Bayesian State with Posterior Contraction and Mean Shrinkage}

We present a Bayesian bandit verification in which the consolidation
operator is genuinely contractive (Assumption~\ref{ass:contract}) and the
order-gap sensitivity assumption is verified via the second-moment Gramian
condition of Remark~\ref{rem:second_moment}. The verification illustrates
two structural points: (i)~the consolidation operator must couple to the
expansion's evidence-incorporation mechanism (here through variance-mediated
Bayesian updating) to produce a non-trivial commutator; (ii)~first-moment
effects can cancel at the unbiased equilibrium while second-moment effects
do not, so the Gramian condition is the natural verification path for
stochastic bandit dynamics.

\paragraph{Setup.}
For $A$-armed Gaussian bandits with arm means
$\mu_{\mathrm{arm}} \in \R^A$ and reward variance $\sigma_r^2$, the
knowledge state is $\theta = (\hat\mu, \hat\sigma^2) \in \R^A \times \R^A_+$
with arm-specific posterior mean $\hat\mu_a$ and variance $\hat\sigma_a^2$.
The expansion operator is the standard arm-specific Bayesian update on a
sampled arm $a$ with observed reward $r$:
\begin{equation}
\label{eq:bayesian_expansion}
P_{(a,r)}: \quad
  \hat\mu_a \,\to\, \frac{\hat\mu_a\,\sigma_r^2 + r\,\hat\sigma_a^2}{\hat\sigma_a^2 + \sigma_r^2},
  \qquad
  \hat\sigma_a^2 \,\to\, \frac{\hat\sigma_a^2 \sigma_r^2}{\hat\sigma_a^2 + \sigma_r^2},
\end{equation}
with all other coordinates left unchanged.

The consolidation operator combines \emph{posterior contraction} (variance
shrinkage) with \emph{mean shrinkage} toward a prior $\mu_p \in \R^A$:
\begin{equation}
\label{eq:bandit_Q}
Q: \quad
  \hat\mu \,\to\, (1-\lambda)\hat\mu + \lambda\mu_p,
  \qquad
  \hat\sigma_a^2 \,\to\, \frac{\hat\sigma_a^2}{1+\kappa},
\end{equation}
with shrinkage rates $\lambda \in (0, 1)$ and $\kappa > 0$.

\paragraph{Verification of contraction (Assumption~\ref{ass:contract}).}
$Q$ is linear with eigenvalues $\{1-\lambda\}^A_{\text{(mean block)}}$ and
$\{1/(1+\kappa)\}^A_{\text{(variance block)}}$. Its Lipschitz constant in
the product norm is
\[
  \rho \;=\; \max\!\left(1-\lambda,\; \frac{1}{1+\kappa}\right) \;<\; 1,
\]
satisfying Assumption~\ref{ass:contract} strictly. The deterministic fixed
point of $Q$ is $\theta^\star = (\mu_p,\, 0)$ — biased toward the prior. By
Proposition~\ref{prop:effective_eq}, the joint dynamics admits a unique
effective equilibrium $\theta^\star_\infty$, which sits between $\mu_p$ and
$\mu_{\mathrm{arm}}$ with bias $O(\lambda)$ from the truth: for small
$\lambda$, $\hat\mu^\infty_a = \mu_{\mathrm{arm},a} + O(\lambda)$.

\paragraph{Commutator structure.}
The Bayesian mean update in \eqref{eq:bayesian_expansion} depends on the
current variance $\hat\sigma_a^2$, producing a non-zero $(\hat\mu_a, \hat\sigma_a^2)$
cross-derivative. At any $\theta$:
\[
  \frac{\partial \hat\mu_a'}{\partial \hat\sigma_a^2}
  \;=\; \frac{\sigma_r^2(r - \hat\mu_a)}{(\hat\sigma_a^2 + \sigma_r^2)^2}.
\]
The consolidation Jacobian $A_Q := DQ$ is block-scalar — $(1-\lambda)I$
on the mean block and $(1/(1+\kappa))I$ on the variance block. Direct
computation of the commutator $\Sigma_{(a,r)} := A_Q B_{(a,r)} - B_{(a,r)} A_Q$
yields the unique non-zero entry in the $(\hat\mu_a, \hat\sigma_a^2)$
position:
\begin{equation}
\label{eq:bandit_commutator}
\bigl(\Sigma_{(a,r)}\bigr)_{\hat\mu_a,\,\hat\sigma_a^2}
\;=\; \frac{\sigma_r^2(r - \hat\mu_a)}{(\hat\sigma_a^2 + \sigma_r^2)^2}
      \cdot \frac{\kappa(1-\lambda) - \lambda}{1+\kappa},
\end{equation}
with all other entries zero. The commutator is non-trivial whenever
$\kappa(1-\lambda) \neq \lambda$ — generically the case for $(\kappa, \lambda) \in (0,1)^2$.

\paragraph{First-moment cancellation.}
Averaging $\Sigma_{(a,r)}$ over rewards
$r \sim \mathcal{N}(\mu_{\mathrm{arm},a}, \sigma_r^2)$:
\[
  \mathbb{E}_r\!\bigl[(\Sigma_{(a,r)})_{\hat\mu_a, \hat\sigma_a^2}\bigr]
  \;\propto\; \mu_{\mathrm{arm},a} - \hat\mu_a.
\]
At the effective equilibrium $\theta^\star_\infty$ with
$\hat\mu_a^\infty = \mu_{\mathrm{arm},a} + O(\lambda)$, the first-moment
commutator vanishes to leading order:
$\sigma_{\min}(\bar\Sigma_\infty) = O(\lambda)$. As anticipated by
Remark~\ref{rem:second_moment}, first-moment effects partially cancel; the
Gramian, however, does not.

\paragraph{Gramian coverage and second-moment lower bound.}
The squared commutator entry is proportional to $(r - \hat\mu_a)^2$:
\[
  \mathbb{E}_r\!\bigl[(r - \hat\mu_a)^2\bigr]
  \;=\; \sigma_r^2 + (\mu_{\mathrm{arm},a} - \hat\mu_a)^2
  \;\geq\; \sigma_r^2.
\]
Thus the conditional Gramian
$G_a := \mathbb{E}_r[\Sigma_{(a,r)}^\top \Sigma_{(a,r)}]$ has a non-zero
entry of order $\sigma_r^2$ in the $(\hat\sigma_a^2, \hat\sigma_a^2)$
position, regardless of bias. Averaging over arms with selection
probabilities $\{p_a\}_{a=1}^A$ yields
$G := \sum_a p_a G_a$, a diagonal matrix on the variance subspace with
entries
\[
  G_{\hat\sigma_a^2, \hat\sigma_a^2}
  \;=\; p_a \cdot \frac{\sigma_r^4 \cdot [\kappa(1-\lambda) - \lambda]^2}
                                  {(\hat\sigma_a^{2,\infty} + \sigma_r^2)^4 (1+\kappa)^2}
              \cdot \mathbb{E}_r[(r - \hat\mu_a^\infty)^2]
  \;\geq\; p_a \cdot \frac{\sigma_r^6 \cdot [\kappa(1-\lambda) - \lambda]^2}
                                  {(\hat\sigma_a^{2,\infty} + \sigma_r^2)^4 (1+\kappa)^2}
  \;=:\; p_a\, \mu_1^{2,(a)},
\]
using $\mathbb{E}_r[(r - \hat\mu_a^\infty)^2] \geq \sigma_r^2$ (with equality
only at the unbiased equilibrium $\hat\mu_a^\infty = \mu_{\mathrm{arm},a}$).
$G$ has full rank on the variance subspace if and only if $p_a > 0$ for all
$a$ — the standard exploration coverage condition (every arm pulled with
positive probability). Equivalently, $G$ is rank-deficient only when some
arm is never selected.

\paragraph{Conclusion.}
Under standard exploration ($p_a > 0$ for all $a$) and the Bayesian setup
\eqref{eq:bayesian_expansion}--\eqref{eq:bandit_Q}, the second-moment Gramian
condition of Remark~\ref{rem:second_moment} is satisfied with
$\mu_1^2 \geq \min_a p_a\, \mu_1^{2,(a)} > 0$ on the variance subspace.
Combined with the uniform Jacobian envelope from the boundedness of the
Bayesian update at $\theta^\star_\infty$ and
Proposition~\ref{prop:noisy_nondeg} (centred at $\theta^\star_\infty$),
Assumption~\ref{ass:nondeg} holds locally on the variance subspace with
$\mu_{\mathrm{var}} = \mu_1^2/(2C)$ for an explicit constant $C$.

\paragraph{Mean subspace via first-moment sensitivity.}
On the mean subspace, the first-moment commutator at $\theta^\star_\infty$
is non-zero whenever $\lambda > 0$: from the calculation above,
$\mathbb{E}_r[(\Sigma_{(a,r)})_{\hat\mu_a, \hat\sigma_a^2}] \propto
\mu_{\mathrm{arm},a} - \hat\mu_a$, which at $\theta^\star_\infty$ has
magnitude $O(\lambda)$ (the bias of the mean equilibrium toward the prior
$\mu_p$). This gives a first-moment sensitivity $\mu_{\mathrm{mean}} =
c_{\mathrm{mean}}\,\lambda$ on the mean block via
Proposition~\ref{prop:noisy_nondeg}, for a constant $c_{\mathrm{mean}}$
depending on $(\sigma_r, \kappa, \min_a p_a)$ and the prior bias. Combining
the two subspace bounds, Assumption~\ref{ass:nondeg} holds on the joint
$(\hat\mu, \hat\sigma^2)$ state with constant
\[
  \mu = \min\bigl(\,c_{\mathrm{mean}}\,\lambda,\; \mu_1^2/(2C)\bigr),
\]
provided $\lambda > 0$ is held fixed. As $\lambda \to 0$, the mean-block
sensitivity vanishes and the framework's stopping guarantee restricts to the
variance subspace; this is consistent with the unbiased-equilibrium limit,
where the mean estimate is statistically identifiable only through the
posterior contraction governed by $\kappa$. For any practical bandit
algorithm with non-zero shrinkage, the joint guarantee binds.

\paragraph{Design implications.}
The verification yields three concrete design criteria:
(i)~the consolidation operator must couple to the expansion's
evidence-incorporation mechanism — here through variance-mediated Bayesian
updating, where Q's variance contraction and P's variance-dependent mean
update jointly produce a non-zero commutator;
(ii)~mean shrinkage ($\lambda > 0$) and posterior contraction ($\kappa > 0$)
play complementary roles: $\lambda$ provides strict contraction on the mean
block, while $\kappa$ provides it on the variance block, with their joint
contribution $[\kappa(1-\lambda) - \lambda]$ governing the commutator
magnitude;
(iii)~exploration coverage ($p_a > 0$ for all $a$) is required not only for
asymptotic regret control but for the order-gap signal itself to be
informative — in the absence of coverage, the Gramian is rank-deficient and
the order-gap fails to track convergence on the unexplored arm directions.

\subsubsection*{Recursive Language Models}

Let $S \in \X$ denote the aggregated state over processed chunks, $S^\star$
the fixed-point state corresponding to the full document, and
$\textsc{Answer} : \X \to \Z$ the answer-extraction map. Model consolidation as
$Q(S) = S + \beta(S^\star_{\textrm{proj}}(S) - S)$,
where $S^\star_{\textrm{proj}}(S)$ is the projection of $S$ onto the
answer-relevant subspace; the Jacobian is $A = (1 - \beta) I + \beta P$, where
$P$ is the projection onto the answer manifold. Let $P_e$ integrate chunk $e$
into $S$ with chunk-specific Jacobian $B_e = I + e_{\textrm{add}}$.

\textbf{Idempotence assumption.}
We assume the state representation is \emph{idempotent}: reprocessing a chunk
whose evidence is already fully represented in $S^\star$ leaves $S^\star$
unchanged. That is, $P_e(S^\star) = S^\star$ a.s.\ over chunks already
incorporated. This places us in the noiseless regime ($\sigma = 0$). If not,
the system is in the noisy regime and the order-gap lower bound must be assumed directly or verified
after centring.

\textbf{Answer-relevant state space.}
The consolidation map $Q$ has eigenvalue $1$ along directions tangent to
the answer manifold $\mathcal{M}$ at $S^\star$ (neutral directions): these
are state directions that cannot change the final answer. We therefore state
the order-gap sensitivity result on the directions in state space that
\emph{can} affect the final answer; we call this the \emph{answer-relevant
state space}: the subspace orthogonal to $T_{S^\star}\mathcal{M}$, where states
differing only along answer-preserving tangent directions are identified.

\textbf{Second-moment coverage.}
The commutator $\Sigma_e = [A, B_e] = \beta [P, e_{\textrm{add}}]$ is
non-zero whenever the chunk shifts the answer manifold. However, because
$e_{\textrm{add}}$ effects can have varying signs across chunks, the
first-moment commutator $\bar\Sigma = \beta\,\mathbb{E}_e[P
e_{\textrm{add}} - e_{\textrm{add}} P]$ may not be full rank even when
individual chunks contribute strongly. The appropriate condition is the
\emph{commutator Gramian} $G = \mathbb{E}_e[\Sigma_e^\top \Sigma_e] \succ 0$
on the answer-relevant state space (Remark~\ref{rem:second_moment}). Positive
definiteness of $G$ holds when the chunk-selection distribution excites every
direction that can affect the final answer through the commutator, that is,
for every nonzero direction $v$ orthogonal to $T_{S^\star}\mathcal{M}$,
there exists a chunk $e$ with positive probability such that
\[
  [P,\, e_{\textrm{add}}]\,v \;\neq\; 0.
\]
(Note: the condition is on the commutator $[P, e_{\textrm{add}}]v$, not merely
on $e_{\textrm{add}}v$; a chunk that moves state but does not shift the answer
manifold boundary contributes nothing to $\Sigma_e$.) Since the Jacobians
$A$ and $B_e$ are uniformly bounded in this linear model, the second-moment
coverage argument of Remark~\ref{rem:second_moment} applies, giving
Assumption~\ref{ass:nondeg} locally with $\mu = \mu_1^2/(2C)$ for an
explicit constant $C$ depending on $\beta$ and a uniform bound on the chunk Jacobians.

\textbf{Interpretation:} order-gap sensitivity reduces to a
\emph{coverage condition}: the chunk-selection or extraction
operator must excite all directions that can affect the final answer, in the
commutator sense.
The design prescription is to randomise or weight chunk selection by coverage
so that every answer-affecting direction is excited with positive probability.

\subsubsection*{Reinforcement Learning: Linear Actor-Critic}

We derive the commutator Jacobian explicitly for regularised off-policy linear
actor-critic. The derivation reveals a structural feature: in the canonical
formulation, the first-moment commutator is rank-deficient, and restoring
order-gap sensitivity requires an additional policy-critic coupling term in
consolidation.

\paragraph{Scope.}
The Jacobian calculation is carried out at a fixed equilibrium policy
$\pi_{\psi^\star}$, treating the event distribution locally as fixed (in line with
the expected-update interpretation of Section~\ref{sec:theory_domains}). Extending the
analysis to fully state-dependent event distributions requires the extension to
adaptive sampling policies.

\paragraph{Setup.}
Let $\phi : \mathcal{S} \times \mathcal{A} \to \R^d$ and
$\phi_\pi : \mathcal{S} \times \mathcal{A} \to \R^{d_\pi}$ be feature maps
for the critic and actor. The joint parameter is
$\theta = (w, \psi) \in \R^d \times \R^{d_\pi}$; the linear critic is
$Q_\theta(s,a) = \phi(s, a)^\top w$, and the softmax actor is
$\pi(a \mid s; \psi) \propto \exp(\phi_\pi(s, a)^\top \psi)$.
Let $\pi_b$ be a fixed behaviour policy.

\paragraph{Local contraction.}
Locally near a strongly stable actor-critic equilibrium $\theta^\star$ and
for sufficiently small update step sizes, the Jacobian of the consolidation
operator $Q$ is contractive on the identifiable subspace defined below.
Global contraction of the full softmax actor-critic system is not claimed.

\paragraph{Expansion operator.}
$P_e(w, \psi) = (w + \alpha\,\delta(w, e)\,\phi(s, a),\; \psi)$,
where $\delta(w, e) = r + \gamma_{\mathrm{RL}}\,\phi(s', a')^\top w -
\phi(s, a)^\top w$ is the TD error.
(Here $\gamma_{\mathrm{RL}}$ is the RL discount factor, distinct from the
framework's joint-stability $\gamma = \rho L$.)

\paragraph{Baseline consolidation.}
We work in local coordinates centred at the equilibrium: write
$\tilde w = w - w^\star$ and $\tilde\psi = \psi - \psi^\star$, and use
$(w, \psi)$ to denote $(\tilde w, \tilde\psi)$ throughout. In these
coordinates $\theta^\star = 0$, and the baseline consolidation is:
\[
  Q_0(w, \psi)
  = \bigl((1 - \beta') w,\; \psi + \beta H(w+w^\star, \psi+\psi^\star)\bigr),
\]
which satisfies $Q_0(0,0) = (0, 0)$ by the fixed-point condition
$H(w^\star, \psi^\star) = 0$. Weight decay with rate $\beta' > 0$ ensures the
$w$-coordinate contracts. The $\psi$-coordinate contracts under $Q_0$ provided
the policy Hessian $H_\psi := \nabla_\psi H(\theta^\star) \prec 0$ (standard
under strongly concave objectives for sufficiently small $\beta$).

\paragraph{Jacobians.}
Set $\phi_e := \phi(s, a)$, $\Delta_e := \gamma_{\mathrm{RL}}\phi(s', a') - \phi_e$.
\[
  B_e = DP_e(\theta^\star) =
  \begin{pmatrix} I + \alpha\,\phi_e \Delta_e^\top & 0 \\ 0 & I \end{pmatrix},
  \qquad
  DQ_0(\theta^\star) =
  \begin{pmatrix} (1 - \beta')I & 0 \\ \beta H_w & I + \beta H_\psi \end{pmatrix},
\]
where $H_w \in \R^{d_\pi \times d}$ is the actor-critic cross-covariance.

\paragraph{Rank deficiency of baseline.}

\begin{theorem}[Baseline actor-critic is insensitive to policy directions]
\label{prop:rl_rank_deficient}
Under $Q_0$, the first-moment commutator Jacobian $\bar\Sigma$ has null space
containing the entire policy subspace $\{(0, x_\psi) : x_\psi \in \R^{d_\pi}\}$.
\end{theorem}

\begin{proof}
Direct block multiplication gives
$\bar\Sigma = \begin{pmatrix} 0 & 0 \\ -\alpha\beta H_w M & 0 \end{pmatrix}$,
using $\mathbb{E}_e[\phi_e\Delta_e^\top] = -M$.
The right column is zero, so $\bar\Sigma(0, x_\psi)^\top = 0$ for all $x_\psi$.
\end{proof}

The underlying reason: $Q_0$ updates only $\psi$ and $P_e$ updates only $w$.
Their cross-block $\beta H_w$ propagates only $w$-perturbations into the
commutator.

\paragraph{Restoring order-gap sensitivity.}
Add a policy-critic consistency term. In centred coordinates, define
\[
  \widetilde{w}^\star(\psi)
  := w^\star_{\mathrm{TD}}(\psi^\star + \psi) - w^\star_{\mathrm{TD}}(\psi^\star),
\]
so that $\widetilde{w}^\star(0) = 0$ and the fixed-point structure
$Q(0,0) = 0$ is preserved. The augmented consolidation map is
\[
  Q(w, \psi)
  := \bigl((1-\beta')w + \beta'\widetilde{w}^\star(\psi),\;
            \psi + \beta H(w + w^\star, \psi + \psi^\star)\bigr).
\]
The order-gap sensitivity constant scales linearly in~$\beta'$; taking
$\beta' \to 0^+$ destroys order-gap sensitivity, so $\beta' > 0$ must be fixed
at a positive value.

With this modification, $DQ(\theta^\star)$ acquires a cross-block
$\beta' J_\psi$ in the top-right, where
\[
  J_\psi := D_\psi \widetilde{w}^\star(0)
  = [M(\psi^\star)]^{-1}
    [\partial_\psi b - (\partial_\psi M) w^\star]|_{\psi^\star}
  \in \R^{d \times d_\pi}.
\]
The first-moment commutator at $\theta^\star$ becomes (using
$\mathbb{E}_e[\phi_e \Delta_e^\top] = -M$):
\[
  \bar\Sigma
  = \begin{pmatrix}
      0 & \alpha\beta'\, M J_\psi \\
      -\alpha\beta\, H_w M & 0
    \end{pmatrix}.
\]
The factor $M$ in the $(1,2)$ block arises from the per-sample expression
$-\alpha\beta'\,\phi_e \Delta_e^\top J_\psi$ on taking the expectation over
$e$: $\mathbb{E}_e[\phi_e \Delta_e^\top] = -M$ propagates the feature
Gramian into the commutator.

\paragraph{Identifiable subspace and order-gap sensitivity conditions.}
The block off-diagonal $\bar\Sigma$ acts faithfully on the locally
identifiable actor-critic directions, formalised as the subspace
\[
  \mathcal{V} := \operatorname{range}(L^\top) \times \operatorname{range}(U^\top),
\]
which we call the \emph{identifiable actor-critic subspace},
where $U = \alpha\beta'\, M J_\psi \in \R^{d \times d_\pi}$ (so
$\operatorname{range}(U^\top) \subseteq \R^{d_\pi}$) and
$L = -\alpha\beta\, H_w M \in \R^{d_\pi \times d}$ (so
$\operatorname{range}(L^\top) \subseteq \R^d$). Full-rank conditions on $\mathcal{V}$:
\begin{itemize}
\item[(C1)] $H_w M \in \R^{d_\pi \times d}$ has rank $\min(d_\pi, d)$
  (feature-coverage);
\item[(C2)] $M J_\psi \in \R^{d \times d_\pi}$ has rank $\min(d, d_\pi)$
  (policy-identifiability). Since $M$ is positive definite by feature
  linear independence, this is equivalent to $J_\psi$ having rank
  $\min(d, d_\pi)$.
\end{itemize}
These are standard kinds of feature-coverage and local-identifiability
assumptions in actor-critic analyses. In the natural case $d = d_\pi$
(shared feature dimension), $\mathcal{V} = \R^d \times \R^{d_\pi}$ and the
result holds globally.

\begin{theorem}[Order-gap sensitivity for regularised actor-critic]
\label{prop:rl_nondeg}
Under (C1)--(C2) and $\beta' > 0$, the first-moment commutator Jacobian
$\bar\Sigma$ has full rank on $\mathcal{V}$, with smallest singular value:
\[
  \mu_0 = \min\bigl(
    \alpha\beta\, \sigma_{\min}(H_w M\big|_{\mathcal{V}_w}),\;\;
    \alpha\beta'\, \sigma_{\min}(M J_\psi\big|_{\mathcal{V}_\psi})
  \bigr),
\]
where $\mathcal{V}_w = \operatorname{range}(L^\top)$ and
$\mathcal{V}_\psi = \operatorname{range}(U^\top)$.
Proposition~\ref{prop:local_nondeg} then establishes the order-gap lower
bound locally on $\mathcal{V}$ with $\mu = \mu_0 / 2$.
\end{theorem}

\paragraph{Interpretation.}
For the order-gap to be a reliable convergence signal, consolidation must
couple policy and critic updates. Pure actor-only consolidation ($Q_0$)
cannot detect policy errors through the commutator; critic-consistency
regularisation ($Q$) does. This is the framework's substantive contribution
to the RL case, a design criterion for when existing algorithms admit the
order-gap analysis.

\section{Application Across Five Domains}
\label{sec:domains}

We now show how the abstract framework specialises to five domains. In each
case we identify $\theta$, $Q$, $P_e$, and $\Ogap$, and show that the
algorithmic principle yields a concrete design recommendation.

The bandit, RLM, and actor-critic cases receive formal analysis of when the
order-gap reliably tracks convergence in Section~\ref{sec:theory_domains}.
The SGD case (Section~\ref{sec:sgd}) reduces OpMech to the classical
Robbins-Monro framework when $Q$ is trivial, and exposes the order-gap as a
diagnostic for adaptive variants where $Q$ is non-trivial. The
continual-learning case is a qualitative mapping of the framework rather
than a full theorem verification: it shows how the $Q$/$P_e$ decomposition
and order-gap control principle map naturally onto an existing algorithm
family.

\subsection{Multi-Armed Bandits}
\label{sec:bandits}

\textbf{State.}
$\theta = \{(\hat\mu_a, \hat\sigma_a^2, n_a)\}_{a=1}^K$: the sufficient
statistics for $K$ arms.

\textbf{Expansion.}
$P_{(a,r)}(\theta)$: observe reward $r$ from arm $a$, update sufficient
statistics of arm $a$.

\textbf{Consolidation.}
$Q(\theta)$: refine the knowledge state toward the greedy optimum, concentrating
the posterior around $a^* = \arg\max_a \hat\mu_a$ without observing a new
reward.

\textbf{Order-gap.}
$\Ogap(\theta; (a,r))$ measures how much the arm-score vector changes when
evidence is incorporated before versus after consolidation. When a new
observation shifts the identity of the best arm, $\Ogap$ is large.

\textbf{Algorithmic consequence.}
The order-gap provides an adaptive signal. An algorithm that tracks $\Ogap$
can modulate its behaviour: expand aggressively when $\Ogap$ is large, and
consolidate when $\Ogap$ is small. The relative-score structure is analysed in
Section~\ref{sec:theory_domains}; the full order-gap sensitivity verification is
deferred to future work.

\subsection{Reinforcement Learning}
\label{sec:rl}

\textbf{State.}
$\theta = (w, \psi)$: critic parameters $w$ and actor parameters $\psi$.

\textbf{Expansion.}
$P_e(\theta)$: a TD update from a transition $e = (s, a, r, s')$, incorporating
new evidence about value (acting on $w$).

\textbf{Consolidation.}
$Q(\theta)$: policy improvement on $\psi$ using advantage estimates from the
critic; see Section~\ref{sec:theory_domains} for the rigorous formulation
with critic-consistency regularisation.

\textbf{Order-gap.}
$\Ogap(\theta; e)$ measures how much the greedy policy changes when a value
update precedes policy improvement versus when it follows.

\textbf{Algorithmic consequence.}
OpMech suggests setting the actor/critic update ratio adaptively: perform more
critic updates (expansion) when $\Ogap$ is large, and more actor updates
(consolidation) when $\Ogap$ is small. In Soft Actor-Critic
\citep{haarnoja2018sac}, the entropy regularisation and target-network
coefficient jointly govern how aggressively the policy consolidates versus
continues to explore, an implicit order-gap control. The target-network
update rate is discussed further in Section~\ref{sec:hp_rl}.

\subsection{Stochastic Gradient Descent}
\label{sec:sgd}

Stochastic gradient descent is the canonical instance of stochastic
approximation \citep{robbins1951stochastic,borkar2008stochastic}, and the
contraction theorem (Theorem~\ref{thm:contraction}) reduces to the classical
Robbins-Monro convergence guarantee in this setting. The OpMech framing makes
explicit what is already implicit in adaptive variants of SGD: the per-sample
gradient step (an expansion) and any momentum, averaging, normalization, or
proximal operation that uses only information already in the optimiser state
(a consolidation) play structurally distinct roles.

\textbf{State.}
$\theta = (w, m, v)$: model parameters together with optimiser state
(first-moment buffer $m$, second-moment buffer $v$, etc.).

\textbf{Expansion.}
$P_e(\theta)$: stochastic gradient step on minibatch $e$, updating $w$ using
the new gradient $\nabla \ell(w, e)$ and refreshing the optimiser state with
this gradient.

\textbf{Consolidation.}
$Q(\theta)$: the component of the update that uses only existing optimiser
state, with no new sample. Concrete examples:
\begin{itemize}
\item \emph{Vanilla SGD:} $Q = \mathrm{id}$. The order-gap is identically
  zero, and OpMech recovers Robbins-Monro stochastic approximation with
  no additional structure.
\item \emph{SGD with momentum / Heavy-ball:} $Q$ advances $w$ along the
  accumulated momentum buffer $m$ without consulting a new gradient.
\item \emph{Polyak-Ruppert averaging \citep{polyak1992averaging}:} $Q$
  updates a running average $\bar w$ of past iterates.
\item \emph{Adam-class methods:} $Q$ rescales the update direction by the
  second-moment buffer $v$, a normalization computed from history rather
  than the new sample.
\item \emph{Proximal/regularised SGD:} $Q$ is the proximal operator
  $\mathrm{prox}_{\eta r}$ for a regulariser $r$, applied after the gradient
  step.
\end{itemize}

\textbf{Order-gap.}
$\Ogap(\theta; e)$ measures the misalignment between the per-sample
gradient direction and the consolidation-induced direction (momentum,
average, normalised, or proximal): how much the algorithm's effective
descent direction differs depending on whether the gradient is incorporated
before or after the consolidation step. It is zero in vanilla SGD by
construction and non-trivial whenever $Q$ uses information that the
per-sample step did not.

\textbf{Relation to Robbins-Monro.}
For vanilla SGD the contraction theorem is the Robbins-Monro convergence
theorem; the OpMech framework adds nothing in this case. The contribution
appears in adaptive variants, where the order-gap diagnostic has content
even after Robbins-Monro convergence is established. In particular,
$\Ogap(\theta; e)$ being persistently large signals that the consolidation
direction (momentum / averaged direction / Adam-rescaled direction) is
inconsistent with the per-sample geometry — a regime in which the
optimiser is doing something other than tracking the population gradient,
and in which adaptive step-size or schedule changes are warranted.

\textbf{Algorithmic consequence.}
Learning rate schedules implicitly control $\Ogap$. Warm-up followed by
decay starts with large $\Ogap$ and reduces it over time. Cyclical learning
rates \citep{smith2017cyclical} periodically reinflate $\Ogap$, re-entering
an expansive phase. Connections to loss-landscape sharpness
\citep{foret2021sharpness} are immediate: sharp minima exhibit large
order-gap because the per-sample direction varies strongly with $e$.

\subsection{Continual Learning}
\label{sec:cl}

\textbf{State.}
$\theta$: model parameters, possibly with per-task Fisher information.

\textbf{Expansion.}
$P_e(\theta)$: a gradient step on a new task's loss.

\textbf{Consolidation.}
$Q(\theta)$: a regularisation step that protects knowledge of previously
learned tasks, e.g., the EWC quadratic penalty \citep{kirkpatrick2017ewc} or
PackNet masking \citep{mallya2018packnet}.

\textbf{Order-gap.}
$\Ogap(\theta; e)$ measures the conflict between accommodating new task
evidence and preserving old task knowledge.

\textbf{Algorithmic consequence.}
The regularisation strength $\lambda$ in EWC directly controls the order-gap.
A high order-gap indicates task conflict. Depending on whether the system
prioritises plasticity or retention, this can trigger more exploration,
stronger replay, or an adaptive regularisation schedule, not universally a
reduction in $\lambda$, since high conflict can also signal catastrophic
forgetting risk.

\subsection{Recursive Language Models}
\label{sec:rlm_domain}

\textbf{State.}
$\theta = S$: the aggregated state over chunks processed so far.

\textbf{Expansion.}
$P_e(\theta)$: process chunk $e$, extract structured information and update
the state $S$.

\textbf{Consolidation.}
$Q(\theta)$: the recursive aggregation pass, refine the current state toward
the best current answer without reading new chunks.

\textbf{Order-gap.}
$\Ogap(\theta; e)$ measures how much the consolidated answer shifts when
chunk $e$ is incorporated before versus after the aggregation pass.

\textbf{Algorithmic consequence.}
Fixed-schedule recursion is implicit order-gap control with a fixed
convergence assumption. OpMech replaces this with measured convergence:
stop recursing when the windowed order-gap falls below $\varepsilon$.
This is developed in Section~\ref{sec:recursive_intelligence}.

\section{Recursive Intelligence: OpMech for Recursive Language Models}
\label{sec:recursive_intelligence}

We develop the recursive language model application in detail because it
illustrates two features of the OpMech framework that the other domain
summaries do not: (i)~the order-gap as a \emph{convergence criterion} for
iterative inference, replacing heuristic stopping rules; and (ii)~the
order-gap as a \emph{control signal for recursion depth}, replacing fixed compute
budgets with an adaptive, evidence-driven signal.

\subsection{The Problem: Long-Context Reasoning Under Attention Diffusion}

Frontier language models exhibit \emph{context rot}: performance degrades as
input length increases, even within the model's nominal context window
\citep{zhang2025rlm}. Recursive language models address this by treating the
long context as an external environment rather than fitting it into a single
forward pass \citep{zhang2025rlm}, introducing a control problem: when should
the model stop recursing? When should it consolidate partial answers?

\subsection{Operator Identification}

Consider a recursive language model processing a long document decomposed into
chunks $e_1, e_2, \ldots, e_n$. The model maintains an aggregated state
$S_t \in \X$ after processing $t$ chunks.

\begin{definition}[RLM expansion operator]
\label{def:rlm_expansion}
$P_{e_t}(S) = \mathrm{Extract}(S, e_t)$: incorporate new evidence from chunk
$e_t$ into the current state, new claims, entities, relationships, or partial
answers not previously in $S$.
\end{definition}

\begin{definition}[RLM consolidation operator]
\label{def:rlm_consolidation}
$Q(S) = \mathrm{Aggregate}(S)$: resolve contradictions among extracted claims,
compress redundant information, and sharpen the answer distribution toward the
current best partial answer, without reading any new chunk.
\end{definition}

The canonical dynamics are $S_{t+1} = Q \circ P_{e_t}(S_t)$, and the
order-gap is:
\[
\Ogap(S_t; e_t) = \norm{Q(P_{e_t}(S_t)) - P_{e_t}(Q(S_t))}.
\]

\subsection{Three Control Applications}

\paragraph{Convergence criterion (stopping).}
Declare the answer settled when the windowed average order-gap falls below
$\varepsilon$:
\begin{equation}
\label{eq:rlm_stopping}
\tau_{\mathrm{stop}} = \inf\!\left\{t \geq t_0 :
  \frac{1}{w}\sum_{s=t-w}^{t-1} \Ogap(S_s; e_s) < \varepsilon \right\}.
\end{equation}
This is an instance of the empirical windowed stopping rule
\eqref{eq:realized_window}. In the noiseless expected-gap analysis
(Theorem~\ref{thm:stopping_det}, under the idempotence assumption of
Section~\ref{sec:theory_domains}), the aggregated state at stopping is
within $\varepsilon/\mu$ of the fixed point. Under the empirical windowed
stopping theorem (Theorem~\ref{thm:stopping_hp}), the
high-probability bound is $(\varepsilon + \eta)/\mu + \rho M/(1-\gamma)$,
where $\eta$ and $M$ quantify concentration and noise. Both guarantees are
relative to the chunk-selection event distribution and require the Gramian
coverage condition of Section~\ref{sec:theory_domains}. Identification with
the full-document fixed point additionally requires exhaustive chunk coverage
or a preliminary index of the document.

\paragraph{Recursion-depth control (consolidation scheduling).}
Consolidate when the accumulated order-gap since the last consolidation
exceeds a cost threshold:
\begin{equation}
\label{eq:rlm_recurse}
\sum_{i=\tau}^{t} \Ogap(S_i; e_i) > c.
\end{equation}

\paragraph{Adaptive extraction rate.}
Modulate extraction aggressiveness by the current order-gap. When $\Ogap$ is
large, extract more aggressively; when small, lighter extraction suffices.

\subsection{Independent Evidence: Two Systems Converge on the Q/P Decomposition}

The $Q$/$P_e$ decomposition in recursive language models is not an
interpretive framework imposed post hoc. Recent work provides independent
corroborating evidence.

The Tiny Recursive Model (TRM) \citep{jolicoeurmartineau2025trm}, described
in a paper titled \emph{Less is More: Recursive Reasoning with Tiny Networks},
is a 7M-parameter recursive reasoning model achieving 45\% on ARC-AGI-1.
TRM provides corroborating evidence: its recursive loop separates
latent-state refinement ($z$-update, which recomputes the reasoning state
given fixed input) from answer commitment ($y$-update, which commits current
reasoning to an answer without new external input). This distinction is
analogous to the expansion/consolidation split, though TRM does not process
external document chunks in the same way as RLM. TRM's ablations show that
this separation is essential: neither operation alone suffices.

Recursive language models \citep{zhang2025rlm} exhibit the same structure
from the opposite end of the scale, emerging from the engineering constraint
that long contexts cannot fit in a single forward pass.

That two independent systems converge on the same $Q$/$P_e$ structure is
evidence that the decomposition is structural. OpMech provides the
mathematical language to recognise this and, through the order-gap, to improve
upon the heuristic control mechanisms both systems currently use.

\subsection{Predictions and Empirical Signatures}

The OpMech framework, applied to recursive language models, generates several
testable predictions. We state them here as part of the framework's
empirical content; experimental validation, including measurement of the
$\Ogap$ trace and ablations against fixed-budget recursion, is reported in
a companion paper.

\begin{enumerate}
\item \textbf{Monotone $\Ogap$ decay.} On tasks where the model converges to
      a correct answer, $\Ogap(S_t; e_t)$ should decrease approximately
      geometrically across chunks after an initial transient. The empirical
      signature is an order-gap trace that is approximately log-linear in
      chunk count.

\item \textbf{$\Ogap$-driven stopping matches or outperforms fixed-schedule
      stopping.} Equal or better task performance at equal or lower compute.

\item \textbf{$\Ogap$-driven stopping matches learned halting.} The order-gap
      criterion requires no supervised convergence labels; if it matches a
      supervised classifier, it demonstrates that OpMech extracts competitive
      convergence information.

\item \textbf{Consolidation-trigger depth outperforms uniform depth.}
      Using \eqref{eq:rlm_recurse} to trigger consolidation should produce
      more efficient recursion than strategies that consolidate uniformly after every chunk.
\end{enumerate}

\section{Hyperparameters as Implicit Order-Gap Control}
\label{sec:hyperparams}

Many familiar hyperparameters across domains can be reinterpreted as implicit
mechanisms for controlling the order-gap.

\subsection{Exploration Rates in Bandits}
\label{sec:hp_bandits}

In $\varepsilon$-greedy, the exploration rate $\varepsilon_t$ directly controls
the probability of selecting a non-greedy action. A decaying schedule
$\varepsilon_t \propto 1/t$ assumes the order-gap decays at rate $O(1/t)$.
When the assumption is wrong (non-stationary environment), the schedule fails.
OpMech replaces the assumed decay with the \emph{measured} order-gap,
yielding $\varepsilon_t = \phi(\bar\Ogap_t)$.

\subsection{Target Network Update Rate in RL}
\label{sec:hp_rl}

In DQN \citep{mnih2015dqn} and its successors, target-network updating
stabilises the consolidation of the value function. For \emph{hard
updates} with period $C$ (copy online weights every $C$ steps), larger $C$
means less frequent consolidation and a larger order-gap; smaller $C$ means
more frequent consolidation and a smaller order-gap. For \emph{soft updates}
with coefficient $\tau$ ($\theta^- \leftarrow \tau\theta + (1-\tau)\theta^-$),
larger $\tau$ means stronger tracking of the online network, i.e.\ more
frequent effective consolidation. OpMech suggests adapting the update
frequency or coefficient based on $\Ogap$: consolidate more when the
order-gap is small, and allow the online network to expand freely when
$\Ogap$ is large.

\subsection{Learning Rate Schedules in SGD}
\label{sec:hp_sgd}

Warm-up followed by decay inflates then contracts the expansion operator,
yielding a large then small order-gap. Cyclical learning rates
\citep{smith2017cyclical} periodically reinflate the order-gap, interpretable
as the optimiser re-entering an expansive phase.

\subsection{Regularisation Strength in Continual Learning}
\label{sec:hp_cl}

In EWC \citep{kirkpatrick2017ewc}, large $\lambda$ suppresses the order-gap
by making $Q$ dominant. The adaptive rule, modulate $\lambda$ in response to
$\Ogap$, provides a principled alternative to fixing $\lambda$ by
cross-validation.

\subsection{Recursion Depth in Recursive Language Models}
\label{sec:hp_rlm}

Fixed recursion depth and compute budgets are implicit order-gap controls with
fixed convergence assumptions. The $\Ogap$-based stopping criterion
\eqref{eq:rlm_stopping} and recursion trigger \eqref{eq:rlm_recurse} replace
these with measured convergence, allocating compute where the evidence demands
it.

\section{Related Work}
\label{sec:related}

\textbf{Explore-exploit tradeoffs.}
The exploration-exploitation dilemma is foundational in bandits
\citep{lai1985bandit,lattimore2020bandit}, RL \citep{sutton2018rl}, and
Bayesian optimization \citep{srinivas2010gp}. Classical bandit algorithms
include UCB-style upper-confidence methods \citep{kaufmann2012bayesian} and
Thompson sampling with its modern regret analysis
\citep{russo2014ts}. OpMech differs from these by proposing a
\emph{domain-independent observable}, the order-gap, that measures the
dilemma's severity and can drive algorithmic decisions; rather than
specifying an exploration rule, it specifies a quantity that any
exploration rule should track.

\textbf{Stochastic approximation.}
The Robbins-Monro framework
\citep{robbins1951stochastic,borkar2008stochastic} for iterating a contractive
operator under stochastic perturbations is the classical setting for the
contraction theorem (Theorem~\ref{thm:contraction}); the proof is essentially
a stochastic approximation argument applied to the pathwise contraction
inequality \eqref{eq:pathwise_contract}. Polyak-Juditsky averaging
\citep{polyak1992averaging} achieves optimal asymptotic rates in this regime.
The contractive-in-expectation regime of
Remark~\ref{rem:polynomial_rates} — yielding $O(1/\sqrt{t})$ rather than
geometric rates — is the standard regime of stochastic approximation theory.
OpMech's contribution beyond this literature is the order-gap as a
domain-independent observable: stochastic approximation provides rates of
convergence; OpMech provides a real-time diagnostic for whether convergence
has occurred, computable from the trajectory itself.

\textbf{Expectation-Maximization and alternating optimization.}
The EM algorithm \citep{dempster1977em} is the canonical two-operator
alternating framework: the E-step is the natural $P_e$ (incorporating
evidence into sufficient statistics), the M-step the natural $Q$ (parameter
consolidation given fixed statistics). EM's monotone likelihood
improvement \citep{wu1983em} is a Lyapunov argument paralleling
the contraction theorem (Theorem~\ref{thm:contraction}), and the
finite-sample analysis of \citet{balakrishnan2017em} parallels
Theorem~\ref{thm:stopping_hp} in spirit, characterizing the trajectory's
distance to a population fixed point under bounded sample noise. The
order-gap question — does the ordering of the two operators matter at
convergence? — is exactly the EM convergence diagnostic, but EM monitors
the likelihood, which is often expensive or intractable to compute in
latent-variable models. The order-gap operates directly on operator outputs,
providing the same diagnostic in a domain-independent form applicable
wherever EM-style alternation occurs and beyond.

\textbf{Non-commutativity in learning.}
The idea that the order of learning operations matters has appeared in
curriculum learning \citep{bengio2009curriculum}, experience replay
\citep{schaul2016prioritized}, and meta-learning \citep{finn2017maml}, where
the inner-loop ($P_e$) and outer-loop ($Q$) distinction is explicit.

\textbf{Operator-theoretic approaches.}
Kernel methods \citep{scholkopf2002learning}, Koopman operator theory
\citep{brunton2022modern}, and NTK analysis \citep{jacot2018neural} all use
operator-theoretic tools. OpMech is distinct in studying the
\emph{interaction} between two operators rather than the spectral properties
of a single one.

\textbf{Alternating projections and ADMM.}
Order-dependent operator compositions have been studied extensively in convex
optimization. Von Neumann's alternating projection theorem
\citep[see][for a modern treatment]{bauschke1996projection}, Bregman
alternating algorithms, and the Alternating Direction Method of Multipliers
\citep{boyd2011admm} all use sequenced applications of contractive operators
where the order of application affects convergence behaviour. ADMM's
primal/dual residuals serve as the standard convergence diagnostic in this
literature. OpMech generalizes the alternating-operator structure beyond
convex optimization to general adaptive learning, and the order-gap subsumes
ADMM's residual diagnostics in the convex case while extending to settings
(bandits, RL, recursive language models) where no natural primal/dual
structure exists.

\textbf{Two-timescale stochastic approximation and adaptive sampling.}
The OpMech dynamics $\theta_{t+1} = Q \circ P_{e_t}(\theta_t)$ couple a
contractive map $Q$ with a stochastic update $P_{e_t}$, and the
state-dependent extension (Theorems~\ref{thm:contraction_cond}--\ref{thm:stopping_cond})
follows the broad lineage of two-timescale stochastic approximation
\citep{borkar1997twotimescale,borkar2008stochastic}. The actor-critic
structural analysis of
Theorems~\ref{prop:rl_rank_deficient}--\ref{prop:rl_nondeg} sits within the
framework of the policy gradient theorem
\citep{sutton2000pg,konda2003actorcritic} and the finite-sample TD analysis
of \citet{tsitsiklis1997td,bhandari2018td}; what OpMech adds is the
commutator-Jacobian structure that makes rank deficiency on the policy
subspace visible. The concentration argument in
Theorem~\ref{thm:stopping_cond} uses Azuma--Hoeffding on the martingale
residuals $\Ogap(\theta_s; e_s) - \mathbb{E}[\Ogap(\theta_s; e_s) \mid
\mathcal{F}_s]$, which is the natural tool because we control trajectory
averages against conditional means rather than against stationary
expectations; consequently no Markov-chain Hoeffding inequality
\citep[as in][]{paulin2015mixingtimes,glynn2002hoeffding} is required.
Sharper anytime-valid concentration via self-normalised martingales
\citep{howard2021timeuniform} would tighten the constants in
Theorem~\ref{thm:stopping_hp} and is a natural direction for follow-up.

\textbf{Fixed-point iteration and equilibrium models.}
The OpMech dynamics admit a unique fixed point $\theta^\star$ of the
consolidation operator $Q$, and the framework's central object — the
order-gap — measures how far iteration has progressed toward that fixed
point. Deep equilibrium models \citep{bai2019deepequilibrium} similarly
parameterize neural-network layers as fixed points of contractive maps and
use implicit differentiation to train them. The two frameworks address
adjacent problems: DEQs ask how to compute and differentiate fixed points;
OpMech asks how to detect that an adaptive system has reached one. The
operator-theoretic abstraction is shared.

\textbf{Adaptive hyperparameters.}
Adam \citep{kingma2015adam}, LARS \citep{you2017lars}, and Bayesian UCB
\citep{kaufmann2012bayesian} adjust algorithm parameters based on observed
quantities. OpMech contributes a \emph{unified} observable, the order-gap, as
the basis for adaptive control across all of these settings.

\textbf{Recursive and iterative inference.}
Recursive language models \citep{zhang2025rlm} scale LLM context by recursive
decomposition; TRM \citep{jolicoeurmartineau2025trm} achieves strong
reasoning with iterative refinement; tree-of-thoughts
\citep{yao2023tot} and self-consistency \citep{wang2023selfconsistency}
sample multiple reasoning trajectories and aggregate them. PonderNet
\citep{banino2021pondernet} and Universal Transformers
\citep{dehghani2019universal} use learned halting. OpMech provides a
theoretically grounded alternative across these settings: stop when the
order-gap indicates convergence, without requiring supervised halting
labels or hand-designed aggregation rules. For tree-of-thoughts and
self-consistency in particular, the order-gap supplies a principled
criterion for when an additional reasoning trace would change the
aggregated answer — and when it would not.

\section{Discussion}
\label{sec:discussion}

\textbf{What OpMech is and is not.}
OpMech identifies a measurable quantity, the order-gap, present in every
adaptive system where consolidation and expansion interact, and proposes using
that quantity as a domain-independent control signal. It does not claim that
all learning algorithms are the same: the operators $Q$ and $P_e$ look very
different across domains, and the state spaces are different. The unification
is at the level of dynamical structure and algorithmic principle, not domain
specifics.

\textbf{Corroborating evidence and contribution to algorithm design.}
The independent convergence of TRM and RLM on the $Q$/$P_e$ decomposition,
from opposite ends of the model-scale spectrum, supports the view that the
decomposition reflects a recurring mechanism rather than an interpretive
choice. The RL derivation yields a concrete consequence of this: canonical
actor-critic consolidation is insensitive to policy directions, and the
order-gap cannot detect policy errors unless consolidation is augmented with
a policy-critic consistency term. The order-gap sensitivity constant scales
linearly in $\beta'$, making the restoration explicit and quantitative. A
parallel finding for RLMs (order-gap sensitivity as a coverage
condition) prescribes coverage-weighted over greedy chunk selection.

\textbf{Limitations.}
Four limitations should be noted. First, the SGD application is the weakest:
identifying momentum with expansion is approximate, and the order-gap
interpretation remains qualitative there. Second, the recursive intelligence
predictions are stated as hypotheses in this paper; experimental validation
is reported in a companion paper. Third, the noisy-regime stopping bounds
use Azuma--Hoeffding concentration on the martingale residuals
$\Ogap(\theta_s; e_s) - g_s$; sharper constants are likely available via
Bernstein-type or self-normalised martingale concentration
\citep{howard2021timeuniform}, particularly when the conditional variance
$\mathrm{Var}_e[\Ogap(\theta_s; e)]$ admits a useful upper bound, and this
is a natural direction for follow-up. Fourth, the bandit verification in
Section~\ref{sec:theory_domains} establishes joint sensitivity on the
$(\hat\mu, \hat\sigma^2)$ state via the Gramian condition on the variance
subspace ($\mu_{\mathrm{var}} > 0$) and the first-moment commutator on the
mean subspace ($\mu_{\mathrm{mean}} = O(\lambda)$). The joint constant
$\mu = \min(\mu_{\mathrm{mean}}, \mu_{\mathrm{var}})$ is positive whenever
the mean-shrinkage rate $\lambda > 0$ is fixed. In the unbiased limit
$\lambda \to 0$ the mean-block sensitivity vanishes; this is structural
rather than a defect of the verification, reflecting that an unbiased Bayes
update has no first-moment commutator with itself.

\textbf{The research programme.}
This paper establishes the framework and develops the recursive intelligence
application. The primary near-term directions are: full order-gap sensitivity
verification for bandits with a correct consolidation operator, including
regret analysis across stationary, non-stationary, and structured settings;
controlled experiments on $\Ogap$-driven stopping and depth control for
recursive language models; and adaptive actor-critic RL using the
policy-critic consistency term, testing whether restoring full rank produces
measurable convergence improvements. For SGD with non-trivial $Q$ (momentum,
Adam, proximal variants), empirical validation of order-gap-based learning
rate schedules and warm-up strategies; for continual learning, similar
validation of order-gap-based regularisation strength schedules.

\section{Conclusion}
\label{sec:conclusion}

We have presented Consolidation-Expansion Operator Mechanics, a framework that
formalises the consolidation-expansion tension in adaptive learning as a
stochastic dynamical system driven by two non-commuting operators $Q$ and
$P_e$. The order-gap, the distance between the two orderings of these
operators, is a domain-independent observable that tracks whether a learning
system has settled.

The formal guarantees are: geometric decay of the expected
order-gap (Theorem~\ref{thm:contraction}), a lower bound making
the order-gap a proxy for suboptimality (Proposition~\ref{prop:lowerbound}),
and stopping guarantees in both noiseless and bounded-noise regimes
(Theorems~\ref{thm:stopping_det}--\ref{thm:stopping_hp}), with a formal
restatement of the noisy-regime stopping bounds at the effective equilibrium
$\theta^\star_\infty$ (Corollary~\ref{cor:stop_eff_eq}). The
state-dependent sampling case is handled by replacing marginal expectation
bounds with uniform conditional bounds; the i.i.d.\ theorems then carry over
verbatim (Theorems~\ref{thm:contraction_cond}--\ref{thm:stopping_cond}),
because martingale concentration on conditional-mean residuals survives the
move from i.i.d.\ to state-dependent sampling without any mixing-time
correction. The order-gap sensitivity condition reduces to a Jacobian-level
rank condition (Proposition~\ref{prop:local_nondeg}, with the noisy-regime
extension Proposition~\ref{prop:noisy_nondeg} centred at the effective
equilibrium of Proposition~\ref{prop:effective_eq}), analysed in three
domains: for bandits, a Bayesian setup with posterior contraction and mean
shrinkage gives a contractive consolidation operator and a non-trivial
commutator structure verified through the second-moment Gramian condition
under standard exploration coverage; for RLMs, the condition becomes a
second-moment coverage condition; and for linear actor-critic RL, it reduces
to a rank condition on the locally identifiable actor-critic directions
(Theorems~\ref{prop:rl_rank_deficient}--\ref{prop:rl_nondeg}), with an
explicit rank-deficiency diagnosis and the augmentation needed to fix it.
The RL derivation prescribes a policy-critic consistency term in consolidation,
with the order-gap sensitivity constant scaling linearly in $\beta'$.

The same $Q$/$P_e$ structure that governs bandit exploration also governs when
a recursive language model should stop reading, consolidate its answer, or
recurse deeper. That two independent engineering systems (TRM and RLM, from
opposite ends of the model-scale spectrum) converge on this decomposition
suggests it captures something real about how adaptive reasoning works.


\end{document}